\RequirePackage{fix-cm}
\documentclass[]{interact}

\usepackage[caption=false]{subfig}

\usepackage{mathtools}
\mathtoolsset{showonlyrefs}

\usepackage{adjustbox}
\usepackage{wrapfig}
\usepackage{float}
\usepackage{placeins}
 
\usepackage{algorithm}
\usepackage{algpseudocode}

\usepackage{enumitem}

\usepackage{tikz}
\usetikzlibrary{
  positioning,
  fit,
  backgrounds,
  arrows.meta,
  calc
}

\usepackage[T1]{fontenc}
\usepackage{xcolor}

\usepackage[natbibapa,nodoi]{apacite}
\usepackage[
  colorlinks=true,
  citecolor=blue,
  urlcolor=blue,
  linkcolor=blue
]{hyperref}

\usepackage{commands}

\begin{document}

\articletype{Preprint}

\title{Adversarial observations in state-space models for robust
reinforcement learning}

\author{
\name{
Miguel Santos-Pascual\textsuperscript{a,b}
\thanks{CONTACT Miguel Santos-Pascual.
Email: miguel.santos@icmat.es}
and
David Ríos Insua\textsuperscript{a}
}
\affil{
\textsuperscript{a}Institute of Mathematical Sciences,
Spanish National Research Council, Madrid, Spain;
\textsuperscript{b}Escuela de Doctorado,
Universidad Autónoma de Madrid, Madrid, Spain
}
}

\maketitle

\begin{abstract}
Decision-making under partial or adversarial observability requires
accurate inference of the environment's latent state and its associated
uncertainty. This work analyses adversarial attacks on linear
state-space models, where the attacker alters observations subject to
likelihood constraints that ensure that the perturbations remain
statistically consistent with the observation model. We analyse how
such adversarial yet plausible observations shift inference about
latent states and affect downstream decision-making and the performance
of reinforcement learning agents. In addition, we introduce an online
Bayesian defence based on directional covariance adaptation, which
selectively reduces the influence of observations by comparing their
estimated impact with that of the computed most disruptive direction,
while preserving information in the remaining orthogonal observation
subspace. The proposed framework provides a principled approach to
constructing robust inference and decision-making systems, with direct
relevance to safety-critical applications such as robotics, where
reliable operation under sensor noise, partial failures, and adversarial
conditions is essential.
\end{abstract}


\section{Introduction}
\label{sec:introduction}

Machine learning (ML) systems increasingly support autonomous decision-making in high-stakes settings such as robotics, transportation, finance, homeland security, and critical infrastructure protection. In these domains, robustness and reliability are essential because failures can translate into physical harm, financial loss, or operational breakdown \citep{garcia2015safe}. A recurring weakness is that many ML pipelines implicitly assume that training and deployment data are independent and identically distributed (i.i.d.), even though actual deployments often violate this assumption through sensor drift, changing environments, and distribution shift \citep{murphy2023}. In security-relevant contexts, this problem is amplified because adversaries can deliberately manipulate observations, rewards, or the environment to induce targeted shifts and drive the system towards failure \citep{barreno2006secure,cina2023wild,riosinsua2023adversarial}.

These concerns motivate the relatively recent field of adversarial machine learning (AML), which studies how malicious perturbations can fool learning systems and how to design defenses against them \citep{biggio2018wildpatterns,goodfellow2015explaining}. While AML is well developed in supervised learning, reinforcement learning (RL) poses distinct challenges, since the goal is to degrade sequential decision-making and thereby reduce the agent's long-term return \citep{pinto2017robust,huang2017policyattack}. Attacks in RL are commonly grouped into three broad categories: attacks on the agent's observations, on the environment's reward signal, and on the agent's policy. In this work, we focus on the first category, observation-level attacks. These include observation spoofing or occlusion, training-time state manipulation through poisoning or backdoor attacks, and strategic perturbations that exploit the agent's exploration process to drive it toward unsafe regions of the state space \citep{behzadan2017policyinduction,gleave2019adversarialpolicies,kiourti2020trojdrl,rathbun2025adversarial}.

In parallel, recent progress in RL has also been shaped by a broader shift towards long-memory sequence models \citep{bengio1994longterm, vaswani2017transformer}. Structured State Space Models (Structured SSMs), including S4 \citep{gu2022s4}, S5 \citep{smith2023s5}, and Mamba \citep{gu2023mamba}, provide long-term temporal context with computational profiles that can be attractive for control workloads; see also \citet{somvanshi2024s4mamba} for a broader discussion of S4 and Mamba architectures. Various surveys summarize how these models relate to and differ from Transformers and classical recurrent networks, and why they have become attractive in modern sequence modeling \citep{somvanshi2024s4mamba,vaswani2017transformer}. Nevertheless, strong sequence modeling alone does not guarantee robust decision-making under uncertainty. In partially observable tasks, the way uncertainty is represented and propagated through the history encoder can critically affect the agent’s behavior \citep{samsami2024uncertainty}.

This work connects these threads by studying adversarial perturbations in probabilistic inference for SSMs. We analyze attacks that operate under likelihood constraints, meaning that perturbations are constructed to remain consistent with the agent's observation model. Such perturbations can be difficult to detect because they appear statistically plausible to the typical internal filter or history encoder, allowing deception to propagate through latent-state inference rather than standing out as anomalous noise \citep{fang2019stealthy}. This issue is especially relevant in applications such as tokamak plasma control \citep{degrave2022magnetic}, autonomous drone racing \citep{kaufmann2023champion}, artificial pancreas systems \citep{cameron2011closed,forlenza2018predictive}, offline industrial RL \citep{deng2023offline}, and planetary powered landing \citep{gaudet2020deep}, where inaccurate inference can severely affect downstream decision-making. Beyond control, this perspective extends to related latent-variable inference problems in time-series modeling, including stochastic volatility models in finance, where distinguishing signal from noise is central to forecasting time-varying risk \citep{shephard2005sv}. Motivated by these settings, our goal is threefold: (i) to characterize how likelihood-constrained adversarial shifts can exploit hidden-state inference; (ii) to study their consequences within RL settings; and (iii) to use this understanding to to improve robustness under both aleatoric uncertainty and strategically crafted realistic perturbations. Our contributions are summarized as follows:
\begin{itemize}[itemsep=1pt, topsep=1pt, parsep=0pt, partopsep=0pt, leftmargin=*]
  \item We define likelihood-constrained adversarial attacks in probabilistic SSMs, moving beyond standard AML settings where perturbation constraints are centered on the observed input itself.
  \item We analyze how these perturbations exploit latent-state inference under different attack objectives and modeling assumptions.
  \item We study the resulting RL formulation and connect it to simpler probabilistic inference cases.
  \item We propose an online Bayesian methodology to protect against adversarial perturbations and disruptive observation noise.
\end{itemize}
Lastly, detailed proofs can be found in \autoref{app:affine_representations} and code used to reproduce the experiments is available in  \href{https://github.com/MiguelSantPasc/Code_AdvSSM}{link}.

\section{Related work}

\textbf{Adversarial machine learning and RL} AML studies how intelligent attackers can induce failures in learning systems by manipulating inputs, data, or interaction channels, and how to build models that are robust against such manipulation. A central theme is that actual deployments can deviate from the i.i.d. assumptions that underlie standard ML training and evaluation pipelines, whether due to aleatoric or targeted distribution shift \citep{quionero2009datasetshift,barreno2006secure}. RL broadens this challenge because the learning process is sequential and interactive: perturbations can affect future data collection, so small corruptions accumulate into large behavioral failures over time \citep{pinto2017robust,huang2017policyattack,behzadan2017policyinduction,gleave2019adversarialpolicies,kiourti2020trojdrl,rathbun2025adversarial}. Attacks in RL are commonly studied across several threat models, including perturbations of the agent's observations, manipulation of the environment's reward signal, and attacks that influence or exploit the agent's policy. Observation-level attacks are especially relevant because they go beyond static input corruption: they may involve spoofing or occluding observations, poisoning or backdooring the training process through state manipulation, or using strategic perturbations that exploit exploration to drive the agent toward unsafe regions of the state space \citep{behzadan2017policyinduction,kiourti2020trojdrl,rathbun2025adversarial}. Interaction-level threats, such as adversarial policies in multi-agent or competitive settings, further show that failures can be induced through the dynamics of interaction rather than direct input perturbation alone \citep{gleave2019adversarialpolicies}. Beyond norm-bounded perturbations, distributional robustness offers a complementary perspective for handling model mismatch, for example through Wasserstein-robust filtering formulations \citep{shafieezadehabadeh2018wasserstein}. Closely related to this perspective are stealthy attacks in estimation and control, which remain difficult to detect because they overcome the statistical tests or innovation properties used by filters \citep{fang2019stealthy}.

\textbf{Structured SSMs as efficient long-context sequence models in RL.}
Structured SSMs have become prominent as sequence models that can capture long-term dependencies with favorable computational properties as surveyed in \cite{somvanshi2024s4mamba,vaswani2017transformer}. In RL, SSM-based encoders are increasingly explored as history models for partially observable tasks, including settings that require long temporal context \citep{lu2023structured}. A key question for robust decision-making is how these architectures represent uncertainty, especially when the agent must balance noisy observations against internal belief \citep{samsami2024uncertainty}. This latent dynamics are central to model-based RL. The Dreamer framework \citep{hafner2019dreamer} introduced latent imagination for learning behaviors by fitting a latent dynamics model and optimizing policies using imagined roll outs. This line of work highlights the importance of accurate latent-state inference for control from high-dimensional observations. Complementary research investigates explicit uncertainty representations in state-space layers and their impact on decision-making under partial observability \citep{samsami2024uncertainty}, a setting naturally formulated as a partially observable Markov decision process (POMDP) \citep{kaelbling1998pomdp}.

\textbf{Robustness in SSMs.}
Robust inference in state-space models has been studied mainly through Kalman-filter variants that mitigate misspecified noise, heavy-tailed disturbances, outliers, or corrupted observations. Some approaches adapt process and observation noise online using variational Bayesian approximations \citep{sarkka2009recursive,sarkka2013variational}, while others use heavy-tailed models or robust measurement updates to down-weight anomalous observations \citep{roth2017robust,wang2018robust,li2020robust}. Recent work connects these ideas with generalised Bayesian filtering and online learning in non-stationary environments \citep{duranmartin2024outlier,duran2024unifying}. Additionally, a complementary line studies distributionally robust Kalman filtering, including Wasserstein ambiguity sets, and secure state estimation under sensor attacks \citep{shafieezadehabadeh2018wasserstein,kargin2024distributionally,shoukry2017secure}. These works mainly frame perturbations as outliers, model mismatch, distributional uncertainty, or corrupted sensors, rather than focusing on statistically plausible perturbations that are strategically constructed to manipulate latent-state inference and downstream decisions.

\section{Background}

This section provides background and introduces notation used throughout the paper.
We use bold uppercase letters ($\mathbf{A}$) to denote matrices and calligraphic letters ($\mathcal{X}$) denote sets.
For a square matrix $\mathbf{A}$, $\mathrm{diag}(\mathbf{A})$ denotes the vector of its diagonal entries.
We write $\mathcal{P}(\mathcal{X})$ for the set of probability distributions over $\mathcal{X}$.

\subsection{State Space Models}

We begin with the general formulation of SSM, which refers to a latent Markovian state process and an observation process. Let $s_t \in \mathbb{R}^{d_s}$ denote a latent state and $o_t \in \mathbb{R}^{d_o}$ an observation. Given an action or control input $a_t \in \mathbb{R}^{d_a}$, a SSM is defined by a transition model \eqref{eq:ssm-transition} and an observation model \eqref{eq:ssm-observation} as illustrated in \autoref{fig:ssm_rl_controlled}
\begin{align}
s_t &\sim p_\theta(s_t \mid s_{t-1}, a_{t-1}), \label{eq:ssm-transition}\\
o_t &\sim p_\theta(o_t \mid s_t, a_{t-1}). \label{eq:ssm-observation}
\end{align}

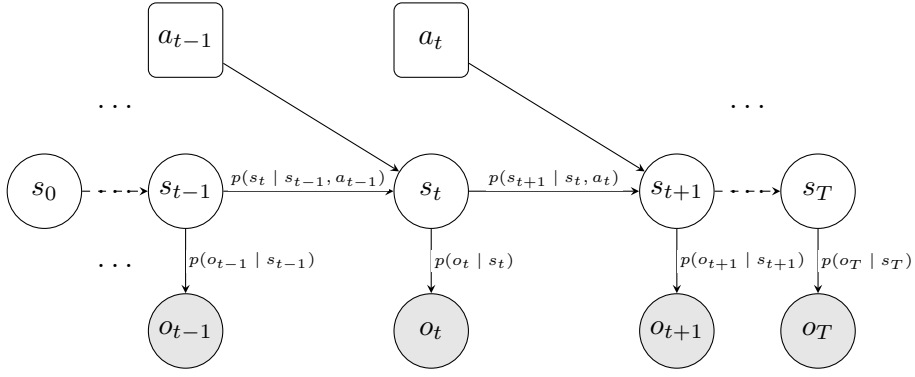
\begin{figure}[t]
\centering
\begin{tikzpicture}[
    >=stealth,
    node distance=2.0cm and 2.9cm,
    latent/.style={circle,draw,minimum size=28pt,inner sep=0pt},
    obs/.style={circle,draw,fill=gray!20,minimum size=28pt,inner sep=0pt},
    act/.style={rectangle,draw,rounded corners=3pt,minimum height=28pt,minimum width=28pt,inner sep=2pt},
    lbl/.style={font=\tiny, fill=white, inner sep=1pt},
    dots/.style={font=\large, fill=white, inner sep=1pt}
]

\def\edgegap{0.15cm}

\node[latent] (s0) {$s_0$};
\node[dots,   right=\edgegap of s0]     (sdotsL) {$\ldots$};
\node[latent, right=\edgegap of sdotsL] (s_prev) {$s_{t-1}$};

\node[latent, right=2.25cm of s_prev] (s) {$s_t$};
\node[latent, right=2.25cm of s]      (s_next) {$s_{t+1}$};

\node[dots,   right=\edgegap of s_next] (sdotsR) {$\ldots$};
\node[latent, right=\edgegap of sdotsR] (sT) {$s_T$};
\node[dots, below=0.85cm of sdotsL] (odotsL) {$\ldots$};
\node[obs,  below=0.85cm of s_prev] (o_prev) {$o_{t-1}$};
\node[obs,  below=0.85cm of s]      (o)      {$o_t$};
\node[obs,  below=0.85cm of s_next] (o_next) {$o_{t+1}$};
\node[dots, below=0.85cm of sdotsR] (odotsR) {$\ldots$};
\node[obs,  below=0.85cm of sT]     (oT)     {$o_T$};

\node[dots, above=1.0cm of sdotsL] (adotsL) {$\ldots$};
\node[act,  above=1.0cm of s_prev] (a_prev) {$a_{t-1}$};
\node[act,  above=1.0cm of s]      (a)      {$a_t$};
\node[dots, above=1.0cm of sdotsR] (adotsR) {$\ldots$};

\draw[->, dashed] (s0) -- (s_prev); 

\draw[->] (s_prev) -- (s)
    node[lbl, above, midway] {$p(s_t\mid s_{t-1},a_{t-1})$};

\draw[->] (s) -- (s_next)
    node[lbl, above, midway] {$p(s_{t+1}\mid s_t,a_t)$};

\draw[->, dashed] (s_next) -- (sT); 

\draw[->] (s_prev) -- (o_prev)
    node[lbl, right, midway] {$p(o_{t-1}\mid s_{t-1})$};

\draw[->] (s) -- (o)
    node[lbl, right, midway] {$p(o_t\mid s_t)$};

\draw[->] (s_next) -- (o_next)
    node[lbl, right, midway] {$p(o_{t+1}\mid s_{t+1})$};

\draw[->] (sT) -- (oT)
    node[lbl, right, midway] {$p(o_T\mid s_T)$};

\draw[->] (a_prev) -- (s);
\draw[->] (a) -- (s_next);

\end{tikzpicture}

\caption{SSM for RL. Latent states $s_t$ evolve according to a probabilistic transition, where the previous action $a_{t-1}$ influences next state. Observations $o_t$ are emitted from latent states. The chain structure continues to earlier and later time steps, as indicated by the transition arrows.}
\label{fig:ssm_rl_controlled}
\end{figure}

\subsection{Dynamic Linear Models and Kalman Filtering}
\label{subsec:dlm_kalman}

A widely used special case of SSM is the \emph{linear Gaussian} SSM, also known as a
Dynamic Linear Model (DLM) \citep{west1997bayesian,petris2009dynamic}.
Let $s_t\in\mathbb{R}^{d_s}$ denote the latent state, $o_t\in\mathbb{R}^{d_o}$ the observation, and
$a_{t-1}\in\mathbb{R}^{d_a}$ a control input applied between $t-1$ and $t$.
The initial state is assigned a Gaussian prior, \(s_0 \sim \mathcal{N}(m_0,P_0),
\) and, for $t\geq 1$, we consider the model:
\begin{align}
s_t &= \mathbf{A}_t s_{t-1} + \mathbf{B}_t a_{t-1} + \omega_t,
\qquad \omega_t \sim \mathcal{N}(\mathbf{0},\mathbf{W}_t),
\label{eq:dlm_transition}\\
o_t &= \mathbf{F}_t s_t + \mathbf{G}_t a_{t-1} + \nu_t,
\qquad \nu_t \sim \mathcal{N}(\mathbf{0},\mathbf{V}_t).
\label{eq:dlm_observation}
\end{align}
In many RL formulations the observation is assumed to depend on the current state only. In that case one sets
$\mathbf{G}_t=\mathbf{0}$ and actions affect observations only through the state dynamics.\footnote{Some references denote the exogenous control input by $u_t$. We adopt the RL convention and write this input as $a_{t-1}$ to emphasize that it corresponds to the action applied between $t-1$ and $t$.}

\textbf{Kalman filtering.}
For \eqref{eq:dlm_transition}--\eqref{eq:dlm_observation}, the involved distributions are Gaussian. Before observing $o_t$, the model propagates the previous posterior through the transition dynamics, giving the predictive belief
\begin{equation}
p(s_t \mid o_{1:t-1},a_{1:t-1})=\mathcal{N}(m_{t|t-1},P_{t|t-1}).
\label{eq:kf_predictive_belief}
\end{equation}
where given $(m_{t-1},P_{t-1})$, the one-step prediction is evolves according to
\begin{equation}
m_{t|t-1} = \mathbf{A}_t m_{t-1} + \mathbf{B}_t a_{t-1},
\qquad
P_{t|t-1} = \mathbf{A}_t P_{t-1}\mathbf{A}_t^\top + \mathbf{W}_t.
\label{eq:kf_prediction}
\end{equation}
This predictive belief induces the one-step predictive distribution of the observation,
\begin{equation}
p(o_t \mid o_{1:t-1},a_{1:t-1})=\mathcal{N}(\hat{o}_t,S_t),
\end{equation}
with
\begin{equation}
\hat{o}_t = \mathbf{F}_t m_{t|t-1} + \mathbf{G}_t a_{t-1},
\qquad
S_t = \mathbf{F}_t P_{t|t-1}\mathbf{F}_t^\top + \mathbf{V}_t.
\label{eq:kf_pred_obs}
\end{equation}
After observing $o_t$, Bayes' rule updates the predictive belief using the likelihood $p(o_t\mid s_t,a_{t-1})$, yielding
\begin{equation}
p(s_t \mid o_{1:t},a_{1:t-1})=\mathcal{N}(m_t,P_t).
\label{eq:kf_belief}
\end{equation}
With innovation $\epsilon_t=o_t-\hat{o}_t$, the update is
\begin{equation}
K_t = P_{t|t-1}\mathbf{F}_t^\top S_t^{-1},
\qquad
m_t = m_{t|t-1} + K_t\,\epsilon_t,
\qquad
P_t = (\mathbf{I}-K_t\mathbf{F}_t)\,P_{t|t-1}.
\label{eq:kf_update}
\end{equation}
Thus, the posterior mean corrects the prediction $m_t$ using the innovation $\epsilon_t=o_t-\hat{o}_t$, with the Kalman gain $K_t$ determining the strength of this correction. Consequently, the estimate relies more on the prediction when observations are noisy and more on the observation when predictive uncertainty is high or observation noise is low.

\textbf{Kalman smoothing.}
When the full sequence $o_{1:T}$ is available, smoothing computes the retrospective posteriors
$p(s_t \mid o_{1:T},a_{1:T-1})$ by combining the forward filter with a backward recursion
\citep{west1997bayesian}.
A standard choice is the \cite{rauch1965maximum} (RTS) smoother.
Let $(m_t,P_t)$ be the filtered moments and $(m_{t+1|t},P_{t+1|t})$ the one-step predictions from
\eqref{eq:kf_prediction}. Define the smoothing gain and terminal condition as
\begin{equation}
J_t = P_t \mathbf{A}_{t+1}^\top P_{t+1|t}^{-1},
\qquad
(m_{T|T},P_{T|T})=(m_T,P_T).
\label{eq:rts_gain}
\end{equation}
For $t=T-1,\dots,1$, the RTS backward recursion is
\begin{equation}
m_{t|T} = m_t + J_t\big(m_{t+1|T}-m_{t+1|t}\big),
\qquad
P_{t|T} = P_t + J_t\big(P_{t+1|T}-P_{t+1|t}\big)J_t^\top.
\label{eq:rts_update}
\end{equation}
Similarly, the RTS smoother corrects the filtered belief using future information. The discrepancy $m_{t+1|T}-m_{t+1|t}$ measures how the smoothed next-state estimate differs from its prediction, while $J_t$ controls how strongly this correction is propagated backward.

\subsection{Reinforcement learning under noisy or partial observability}

A natural way to embed state-space inference within a RL setting is through POMDPs, formally, defined as
\[
\mathcal{M} = (\mathcal{S}, \mathcal{A}, \mathcal{O}, P, O, r, \gamma, \rho_0),
\]
where $\mathcal{S}$ is the (latent) state space, $\mathcal{A}$ is the action space, and $\mathcal{O}$ is the observation space.
The transition kernel $P(\cdot \mid s,a) \in \mathcal{P}(\mathcal{S})$ specifies the environment dynamics, that is, \(s_{t+1} \sim P(\cdot \mid s_t, a_t),
\)
and the observation kernel $O(\cdot \mid s) \in \mathcal{P}(\mathcal{O})$ specifies how observations are generated from states, that is,
\(o_t \sim O(\cdot \mid s_t).
\)
The reward function $r:\mathcal{S}\times\mathcal{A}\to\mathbb{R}$ assigns immediate utility to taking action $a_t$ in state $s_t$, and $\gamma\in(0,1]$ is the discount factor that controls how future rewards are weighted.
The initial state distribution is $s_0 \sim \rho_0$.

Because the agent does not directly observe $s_t$, it selects actions using a history-dependent policy
\[
a_t \sim \pi(\cdot \mid h_t), \qquad h_t = (o_{1:t}, a_{1:t-1}),
\]
which can be summarized through the belief state $s_t\sim p(s_t\mid h_t)$. 

The RL objective is to find a policy $\pi$ that maximizes the expected discounted return
\begin{equation}
J(\pi)
=
\mathbb{E}_{\pi}\!\left[
\sum_{t=0}^{T-1}
\gamma^t r(s_t,a_t)
\right].
\end{equation}
where the expectation is taken over trajectories generated by $\rho_0$, $P$, $O$, and the policy $\pi$. In addition, we consider the \textit{action-value function} (or \textit{$Q$-function}) induced by $\pi$, which evaluates the expected future discounted return after taking action $a_t$ in state $s_t$
\begin{equation}
Q^\pi(s_t,a_t)
=
\mathbb{E}_{\pi}\!\left[
\sum_{k=t}^{T-1}
\gamma^{k-t} r(s_k,a_k)
\,\middle|\,
s_t,a_t
\right].
\end{equation}
The action-value function is closely related to the \textit{state-value function}, which measures the expected return obtained from state $s_t$ when following a policy $\pi$,
\begin{equation}
V^\pi(s_t)
=
\mathbb{E}_{a_t\sim\pi(\cdot\mid s_t)}
\left[
Q^\pi(s_t,a_t)
\right].
\end{equation}
Importantly, the Bellman relation gives
\begin{equation}
\label{eq:q_bellman_rl_attack}
Q^\pi(s_t,a_t)
=
\mathbb E_{s_{t+1}\sim p(\cdot\mid s_t,a_t)}
\left[
r(s_t,a_t,s_{t+1})
+
\gamma V^\pi(s_{t+1})
\right],
\end{equation}

\section{Problem Statement}
\label{sec:problemstatement}
We aim to design attacks against a system whose inference procedure relies on a KF. A \emph{white-box} setting is considered, in which the attacker is assumed to have full knowledge of the underlying state-space model, its parameters, and the inference procedure. The attack consists of modifying a single observation \(o_t\) to induce a controlled change in the inferred latent-state belief, particularly in its posterior mean \(\mu_t\). By altering the estimated state in this way, the attacker seeks to affect the expected value of a target quantity relevant to downstream decision-making.

\textbf{Attack setting and notation.}
Let $o_{1:T}$ be a clean observation sequence and let $a_{1:T-1}$ be the corresponding actions or controls. The attacker perturbs at a single time index $t\in\{1,\dots,T\}$ by replacing $o_t$ with an adversarial observation $o_t^{adv}\in\mathcal{O}$. We write $o_{-t}$ for the set of observations excluding $o_t$, so that the attacked history is $(o_{-t}, o_t^{adv})$. The attack is then propagated through the posterior distribution of the latent state at time $t$.\footnote{Online deployments are obtained as a special case where the attack is applied at the current time step, i.e., $t=T$. In that regime, the relevant posterior reduces to the filtering distribution $p(s_T \mid o_{1:T},a_{1:T-1})$, since no future observations are available and smoothing is not applicable. This is only a particualr case of the general approach.}

\textbf{Optimization objective.}
Let $g:\mathcal{S}\to\mathbb{R}^m$ be a target function of the latent state, and let $d(\cdot,\cdot)$ be a discrepancy measure. The attacker aims to change the posterior expectation of $g(s_t)$ by selecting $o_t^{adv}$
\begin{equation}
\max_{o_t^{adv}\in\mathcal{O}}
\quad
d\!\left(
\mathbb{E}_{s_t\sim p(\cdot \mid o_{1:T},a_{1:T-1})}[g(s_t)],
\,
\mathbb{E}_{s_t\sim p(\cdot \mid o_{-t},o_t^{adv},a_{1:T-1})}[g(s_t)]
\right).
\label{eq:attack-max-shift}
\end{equation}
In some applications, the goal is to steer the hidden state estimate toward a desired target value $M\in\mathbb{R}^m$ rather than simply maximizing change. This yields the targeted formulation
\begin{equation}
\min_{o_t^{adv}\in\mathcal{O}}
\quad
d\!\left(
\mathbb{E}_{s_t\sim p(\cdot \mid o_{-t},o_t^{adv},a_{1:T-1})}[g(s_t)],
M
\right).
\label{eq:attack-targeted}
\end{equation}
Throughout this work, $d$ is typically chosen as a squared Euclidean distance when $m>1$ or an absolute difference when $m=1$, although other choices are possible.

This formulation allows the same mechanism to express different attack objectives by choosing $g$ and $d$ as we now illustrate with several important examples.

\begin{itemize}

\item[A.] \textbf{Attacks on hidden state inference.}

These attacks aim to distort the agent's belief about the latent state by modifying a single observation $o_t$. They are previously detailed \autoref{sec:Methodology}.

\begin{itemize}

\item[(a1)] \textbf{Hidden state point attack.} Take $g(s_t)=s_t$ and $d(\mathbf{x},\mathbf{y})=\|\mathbf{x}-\mathbf{y}\|_2^2$.  
The attacker seeks to maximize the difference between the actual posterior mean and the posterior mean obtained after replacing $o_t$ with $o_t^{adv}$
\begin{equation}
\max_{o_t^{adv}\in\mathcal{O}}
\left\|
\mathbb{E}_{p(\cdot \mid o_{1:T},a_{1:T-1})}[s_t]
-
\mathbb{E}_{p(\cdot \mid o_{-t},o_t^{adv},a_{1:T-1})}[s_t]
\right\|_2^2 .
\label{eq:attack-estimate}
\end{equation}

\item[(a2)] \textbf{Target function point attack.} Take $g(s_t)\in \mathbb{R}$ and $d(\mathbf{x},\mathbf{y})=\|\mathbf{x}-\mathbf{y}\|_2^2$ with target $M\in\mathbb{R}$.  
The attacker attempts to steer the posterior mean of a task-relevant scalar functional toward a specific target value
\begin{equation}
\min_{o_t^{adv}\in\mathcal{O}}
\left(
\mathbb{E}_{p(\cdot \mid o_{-t},o_t^{adv},a_{1:T-1})}[g(s_t)]
- M
\right)^2 .
\label{eq:attack-g-function}
\end{equation}

This formulation is particularly relevant in safety-critical estimation settings where the decision depends on a scalar quantity derived from the latent state rather than on the state itself. 


\end{itemize}

\vspace{0.3cm}
\item[B.] \textbf{Reward-oriented attacks.}

In RL, two different scenarios may arise depending on whether the attack is performed on an \emph{offline dataset} or on a \emph{real-time agent}.

\begin{itemize}

\item[(b1)] \textbf{Offline trajectory attack (data poisoning).}  
If the attacker modifies an observation $o_t$ in a trajectory dataset used for training, the corresponding action $a_t$ is already fixed in the recorded data. In this case, the goal is to distort the expected reward associated with that state-action pair
\[
\min_{o_t^{adv}\in\mathcal{O}}
\;
\mathbb{E}_{s_t\sim p(s \mid o_{-t},o_t^{adv},a_{1:t-1})}
\big[r(s_t,a_t)\big].
\]
This scenario corresponds to manipulating the trajectories that will be used later by an offline RL algorithm to estimate value functions or policies. Note that this problem is equivalent to the one presented in (a2).

\item[(b2)] \textbf{Long-term reward attack}.  
In RL, the attacker does not aim to directly distort the latent state, but rather to manipulate the agent's \emph{perception} of the current state so that the action selected under the attacked belief performs poorly for the \emph{actual} state of the system. The agent then chooses an action according to \(a_T \sim \pi(\cdot \mid s_T^{adv})\), but this action is executed in the real environment, whose actual latent state remains $s_T^0$. Therefore, the attack objective is
\begin{equation}
\min_{o_T^{adv}\in\mathcal O}
\;
\mathbb{E}_{s_T^{adv}\sim p^{adv}(\cdot \mid o_{1:T-1},o_T^{adv},a_{1:T-1})}
\!\left[
\mathbb{E}_{a_T\sim \pi(\cdot \mid s_T^{adv})}
\big[
Q^\pi(s_T^0,a_T)
\big]
\right].
\label{eq:attack-RL}
\end{equation}
which objective captures what would be the actual purpose of the attack: modify the agent's belief about the current state so that the resulting action distribution leads to low reward when applied at the actual state $s_T^0$.


\end{itemize}

\end{itemize}

\textbf{Likelihood and plausibility constraints.}
To avoid trivially detectable attacks, we constrain $o_t^{\mathrm{adv}}$ to remain plausible under the agent's predictive model. This helps ensure that the adversarial observation remains undetected through anomaly-detection systems. This differs from standard adversarial-example settings, where perturbations are usually constrained directly around the observed input \citep{goodfellow2015explaining,madry2018towards}. In our setting, plausibility is expressed directly in probabilistic terms by requiring the predictive likelihood of the adversarial observation to remain above a threshold $\epsilon>0$. That is,
\begin{equation}
p\!\left(o_t^{\mathrm{adv}}\mid o_{-t},a_{-t}\right)
\ge
\epsilon.
\label{eq:likelihood-constraint}
\end{equation}
For numerical optimization, it is often convenient to rewrite \eqref{eq:likelihood-constraint} as a negative log-likelihood bound or, in the Gaussian case, as a Mahalanobis-distance constraint.

\section{Methodology}
\label{sec:Methodology}
In this section, we develop the attacks introduced in \autoref{sec:problemstatement}, deriving their different formulations and examining their main properties and implications. Detailed proofs can be found in \autoref{app:affine_representations}.

\subsection{\textbf{Adversarial attack on the state estimation}}
\label{subsec:lgssm_attack}
In this setting, the goal is to steer the agent's latent-state estimate, represented by the posterior mean, while keeping the perturbed observation plausible under the agent's predictive model. We replace a single measurement $o_t$ with $o_t^{\mathrm{adv}}$ while requiring the perturbation to remain statistically plausible. Concretely, the attack is of the form
\[
\max_{o_t^{\mathrm{adv}}\in\mathcal{O}} \ \mathcal{J}(o_t^{\mathrm{adv}})
\quad \text{s.t.}\quad
p(o_t^{\mathrm{adv}}\mid o_{-t},a_{1:T-1}) \ge \epsilon,
\]
where \(
\mathcal{J}(o_t^{\mathrm{adv}}) =
\left\|
\mathbb{E}_{p(\cdot \mid o_{1:T},a_{1:T-1})}[s_t]
-
\mathbb{E}_{p(\cdot \mid o_{-t},o_t^{\mathrm{adv}},a_{1:T-1})}[s_t]
\right\|_2^2\) as in \eqref{eq:attack-estimate}. 

The first step is to write $\mathcal{J}(o_t^{\mathrm{adv}})$ in closed form, so that the maximization problem can be solved explicitly in terms of $o_t$. To do so, consider the controlled DLM in \eqref{eq:dlm_transition} and \eqref{eq:dlm_observation}, with initial distribution $s_0\sim\mathcal{N}(m_0,P_0)$. Let $K_t$ denote the Kalman gain and define
\begin{equation}
\mathbf{M}_t := (\mathbf{I}-K_t\mathbf{F}_t)\mathbf{A}_t,
\qquad
\mathbf{U}_t := (\mathbf{I}-K_t\mathbf{F}_t)\mathbf{B}_t - K_t\mathbf{G}_t.
\label{eq:MU_def}
\end{equation}
Then the filtered mean recursion can be written in affine form as
\begin{equation}
m_t = \mathbf{M}_t\, m_{t-1} + \mathbf{U}_t\, a_{t-1} + K_t\, o_t.
\label{eq:filter_affine_recursion}
\end{equation}

Starting from the filtering recursion, we aim to express the hidden-state
estimate at a future time $t+k$ in terms of the previous estimate
$m_{t-1}$, the observations, and the actions. This yields an affine
representation that is particularly convenient for attack optimization.

\begin{lemma}[Affine representation of the KF mean]
\label{lem:past_unroll}
Fix $t\in\{1,\ldots,T\}$ and $k\in\{0,\ldots,T-t\}$. Under the controlled DLM, there exist matrices
$\boldsymbol{\Phi}_{t,k}$,
$\boldsymbol{\Gamma}^{(o)}_{t,k,i}$, and
$\boldsymbol{\Gamma}^{(a)}_{t,k,i}$ such that
\begin{equation}
m_{t+k}
=
\boldsymbol{\Phi}_{t,k}m_{t-1}
+
\sum_{i=0}^{k}
\boldsymbol{\Gamma}^{(o)}_{t,k,i}\,o_{t+i}
+
\sum_{i=0}^{k}
\boldsymbol{\Gamma}^{(a)}_{t,k,i}\,a_{t+i-1}.
\label{eq:filtered_affine_compact}
\end{equation}
Therefore, conditional on the past $m_{t-1}$ and the action sequence
$a_{t-1:t+k-1}$, the filtered mean $m_{t+k}$ is an affine function of the
observations $o_{t:t+k}$.
\lemmaend
\end{lemma}
Smoothed estimates are also relevant for offline inference and robust
training, where state estimates may be readjust using future observations.
We therefore extend the representation in
Lemma~\ref{lem:past_unroll} to the RTS smoothed mean $m_{t|T}
=
\mathbb{E}
\left[
s_t
\mid
o_{1:T},a_{0:T-1},m_0
\right].$
\begin{lemma}[Affine representation of the RTS mean]
\label{lem:explicit_o_a_rewrite}
Fix $t\in\{1,\ldots,T\}$. Under the controlled DLM, there exist matrices
$\mathbf H^{(m)}_t$,
$\mathbf H^{(o)}_{t,u}$, and
$\mathbf H^{(a)}_{t,u}$ such that
\begin{equation}
m_{t|T}
=
\mathbf H^{(m)}_t m_{t-1}
+
\sum_{u=t}^{T}
\mathbf H^{(o)}_{t,u}\,o_u
+
\sum_{u=t-1}^{T-1}
\mathbf H^{(a)}_{t,u}\,a_u.
\label{eq:smoothed_affine_compact}
\end{equation}
Therefore, conditional on $m_{t-1}$ and the action sequence
$a_{t-1:T-1}$, the RTS smoothed mean $m_{t|T}$ is an affine function of the
observations $o_{t:T}$.
\lemmaend
\end{lemma}

The second step is to obtain a closed-form expression for the probability
distribution defining the feasible region. In a DLM, the leave-one-out predictive density of the omitted observation is gaussian, \(p(o_t\mid o_{-t},a_{0:T-1})
=
\mathcal N(\hat{o}_{-t},S_{-t}),
\)
so the likelihood constraint \eqref{eq:likelihood-constraint} defines an
explicit ellipsoidal region centered at the leave-one-out predictive mean
$\hat{o}_{-t}$. Following \citet{harrison_west_1991}, the required moments can be recovered directly from the original distribution.

\begin{lemma}[Leave-one-out predictive observation density]
\label{lem:loo_pred_obs_past_future}
Consider the controlled DLM
\eqref{eq:dlm_transition}--\eqref{eq:dlm_observation}, fix
$t\in\{1,\ldots,T\}$, and let
\(
o_{-t}
:=
\{o_1,\ldots,o_{t-1},o_{t+1},\ldots,o_T\}.
\) Then, given \( p(s_t\mid o_{1:T},a_{0:T-1})
=
\mathcal N(m_{t\mid T},P_{t\mid T}).\)
the leave-one-out distribution is,
\begin{align}
p(o_t\mid o_{-t},a_{0:T-1})
&=
\mathcal N(\hat{o}_{-t},S_{-t}),
\label{eq:loo_pred_obs_dist}
\\
S_{-t}
&=
\left[
\mathbf V_t^{-1}
-
\mathbf V_t^{-1}\mathbf F_t
P_{t\mid T}
\mathbf F_t^\top\mathbf V_t^{-1}
\right]^{-1},
\label{eq:loo_pred_obs_cov}
\\
\hat{o}_{-t}
&=
o_t
-
S_{-t}\mathbf V_t^{-1}
\left(
o_t
-
\mathbf F_t m_{t\mid T}
-
\mathbf G_t a_{t-1}
\right).
\label{eq:loo_pred_obs_mean}
\end{align}
\lemmaend
\end{lemma}

Recalling the attack objective, we seek an adversarial observation $o_t^{\mathrm{adv}}$ that maximally changes the inferred hidden-state estimate
\[
\max_{o_t^{\mathrm{adv}}\in\mathcal{O}}
\left\|
\mathbb{E}[s_t\mid o_{1:T},a_{1:T-1}]
-
\mathbb{E}[s_t\mid o_{-t},o_t^{\mathrm{adv}},a_{1:T-1}]
\right\|_2^2
=
\max_{o_t^{\mathrm{adv}}\in\mathcal{O}}
\left\| m_{t|T}-m_{t|T}^{\mathrm{adv}}\right\|_2^2 .
\]
We emphasize that even in this white-box setting, where the attacker has full knowledge of the model and access to the complete observation sequence, the latent state $s_t$ remains unobservable to both the attacker and the defender. Hence, the attack is necessarily defined in terms of the posterior estimate $m_{t|T}$ rather than the actual state itself. Consequently, when the estimation error at time $t$ is large, for instance because of high observation noise, an adversarial perturbation that moves the posterior mean away from its original value may inadvertently move it closer to the latent state. In this sense, the attack targets the agent's belief, not direct access to the underlying state.

Using the affine representation of the smoothed mean from Lemma~\ref{lem:explicit_o_a_rewrite} and the leave-one-out predictive density from Lemma~\ref{lem:loo_pred_obs_past_future}, the attack can be written explicitly through the following ellipsoid-constrained quadratic problem.


\begin{theorem}[State-estimate point attack]
\label{thm:naive_qcqp}
Assume the affine representation of the RTS mean given in
Lemma~\ref{lem:explicit_o_a_rewrite} and define \(
\mathbf R_t
:=
\left(\mathbf H^{(o)}_{t,t}\right)^\top
\mathbf H^{(o)}_{t,t}
\succeq \mathbf 0.\)
For a likelihood threshold satisfying
\[
0<\epsilon
\leq
(2\pi)^{-d_o/2}|S_{-t}|^{-1/2},
\]
the state-estimate point attack is equivalent to the ellipsoid-constrained
quadratic problem
\begin{equation}
\boxed{
\begin{aligned}
\max_{o_t^{\mathrm{adv}}\in\mathbb R^{d_o}}
\quad&
(o_t^{\mathrm{adv}}-o_t)^\top
\mathbf R_t
(o_t^{\mathrm{adv}}-o_t)
\\
\mathrm{s.t.}\quad&
(o_t^{\mathrm{adv}}-\hat o_{-t})^\top
S_{-t}^{-1}
(o_t^{\mathrm{adv}}-\hat o_{-t})
\leq
\rho_\epsilon,
\end{aligned}}
\label{eq:naive_qcqp}
\end{equation}
where
\begin{equation}
\rho_\epsilon
:=
-2\log\epsilon
-d_o\log(2\pi)
-\log|S_{-t}|.
\label{eq:rho_epsilon_definition}
\end{equation}
\lemmaend
\end{theorem}

Problem \eqref{eq:naive_qcqp} is a quadratically constrained quadratic optimization problem of trust-region type. Its global optimum is characterized by the KKT conditions \citep{boyd2004convex,conn2000trust,nocedal2006numerical}. An approach to solve it is described as in \autoref{alg:single_point_attack}.

\begin{algorithm}[ht]
\caption{Likelihood-constrained single-observation attack on hidden state estimate}
\label{alg:single_point_attack}
\begin{algorithmic}[1]
\Statex \textbf{Required:} SSM parameters $\{\mathbf A_u,\mathbf B_u,\mathbf F_u,\mathbf G_u,\mathbf W_u,\mathbf V_u\}_{u=1}^T$ and initial belief $(m_0,P_0)$ (or fixed $s_0=m_0$).
\Statex \textbf{Input:} observations $o_{1:T}$, actions $a_{0:T-1}$, time step $t$, constraint level $\epsilon$ (or $\rho_\epsilon$).
\Statex \textbf{Output:} adversarial observation $o_t^{\mathrm{adv}}$.
\Statex \rule{\linewidth}{0.4pt}

\State Compute $\mathbf H^{(o)}_{t,t}$ and $\mathbf Q_t=(\mathbf H^{(o)}_{t,t})^\top\mathbf H^{(o)}_{t,t}$
from Lemma~\ref{lem:explicit_o_a_rewrite}.
\State Compute $(\hat o_{-t},S_{-t})$ and $\rho_\epsilon$ from Lemma~\ref{lem:loo_pred_obs_past_future}.
\State Solve \eqref{eq:naive_qcqp} via KKT conditions (generalized trust-region form) and return $o_t^{\mathrm{adv}}$.
\end{algorithmic}
\end{algorithm}

\subsection{\textbf{Adversarial attack on a target function estimate}}
\label{subsec:lgssm_attack_on_g}

Consider now the attack to the scalar-target case, where
\(
g:\mathbb R^{d_s}\to\mathbb R,
\) ans \(M\in\mathbb R.
\)
The attacker no longer seeks to directly perturb the smoothed state estimate itself, but rather the posterior expectation of a scalar function of the latent state. Using the same likelihood-based plausibility region as in the previous section ans the objective function \eqref{eq:attack_g_scalar}, the attack problem becomes
\begin{equation}
\boxed{
\begin{aligned}
\min_{o_t^{\mathrm{adv}}\in\mathbb R^{d_o}}
\ & \Big(
\mathbb E_{p(\cdot\mid o_{1:t-1},\,o_t^{\mathrm{adv}},\,o_{t+1:T},\,a_{1:T-1})}[g(s_t)]
-M
\Big)^2 \\
\text{s.t.}\ & (o_t^{\mathrm{adv}}-\hat o_{-t})^\top S_{-t}^{-1}(o_t^{\mathrm{adv}}-\hat o_{-t})
\le \rho_\epsilon .
\end{aligned}}
\label{eq:attack_g_scalar}
\end{equation}
Following Lemma~\ref{lem:explicit_o_a_rewrite}, the smoothed mean is affine in the attacked observation. Isolating the dependence on $o_t$, we can write
\begin{equation}
m_{t|T}(o_t^{\mathrm{adv}})
=
\bar m_{t|T}^{(-t)} + \mathbf H^{(o)}_{t,t}\, o_t^{\mathrm{adv}},
\label{eq:m_tT_affine_attack}
\end{equation}
where
\begin{equation}
\bar m_{t|T}^{(-t)}
=
\mathbf H^{(m)}_t m_{t-1}
+
\sum_{u=t+1}^{T}\mathbf H^{(o)}_{t,u}o_u
+
\sum_{u=t-1}^{T-1}\mathbf H^{(a)}_{t,u}a_u
\label{eq:m_bar_minus_t}
\end{equation}
collects all terms not depending on the attacked observation. Moreover, the smoothed posterior remains Gaussian
\begin{equation}
p(s_t\mid o_{1:t-1},o_t^{\mathrm{adv}},o_{t+1:T},a_{1:T-1})
=
\mathcal N\!\bigl(m_{t|T}(o_t^{\mathrm{adv}}),\,P_{t|T}\bigr).
\label{eq:smoothed_posterior_attacked}
\end{equation}
Importantly, for fixed model parameters, the covariance $P_{t|T}$ does not depend on the realized observation values, and, therefore, it remains constant throughout the optimization. Hence only the smoothed posterior mean changes with the attack.

Using \eqref{eq:smoothed_posterior_attacked}, define
\begin{equation}
\mu_g(o_t^{\mathrm{adv}})
:=
\mathbb E_{\mathcal N(m_{t|T}(o_t^{\mathrm{adv}}),P_{t|T})}[g(s_t)].
\label{eq:mu_g_scalar}
\end{equation}
Then the attack objective becomes simply
\begin{equation}
J(o_t^{\mathrm{adv}})
=
\bigl(\mu_g(o_t^{\mathrm{adv}})-M\bigr)^2.
\label{eq:J_scalar}
\end{equation}
Introducing the Cholesky decomposition,
\begin{equation}
s_t
=
m_{t|T}(o_t^{\mathrm{adv}})+L_t\xi,
\qquad
\xi\sim\mathcal N(0,\mathbf I),
\qquad
L_tL_t^\top=P_{t|T},
\label{eq:reparam_scalar}
\end{equation}
we obtain
\begin{equation}
\mu_g(o_t^{\mathrm{adv}})
=
\mathbb E_\xi\!\left[
g\!\left(
\bar m_{t|T}^{(-t)}+\mathbf H^{(o)}_{t,t}o_t^{\mathrm{adv}}+L_t\xi
\right)
\right].
\label{eq:mu_g_reparam_scalar}
\end{equation}
Assuming that $g$ is differentiable, with gradient $\nabla g(s_t)\in\mathbb R^{d_s}$, differentiation under the expectation of a Gaussian distribution yields
\begin{equation}
\nabla_{o_t^{\mathrm{adv}}}\mu_g(o_t^{\mathrm{adv}})
=
(\mathbf H^{(o)}_{t,t})^\top\,
\mathbb E\!\left[\nabla g(s_t)\right].
\label{eq:grad_mu_g_scalar}
\end{equation}
Therefore,
\begin{equation}
\nabla_{o_t^{\mathrm{adv}}}J(o_t^{\mathrm{adv}})
=
2\bigl(\mu_g(o_t^{\mathrm{adv}})-M\bigr)
(\mathbf H^{(o)}_{t,t})^\top
\mathbb E\!\left[\nabla g(s_t)\right].
\label{eq:grad_J_scalar}
\end{equation}

When an analytic gradient of $g$ is unavailable, $\nabla g$ can be replaced by a finite-difference approximation. In practice, both $\mu_g$ and the gradient term are estimated by Monte Carlo. Given samples $\xi^{(1)},\dots,\xi^{(N)}\sim\mathcal N(0,\mathbf I)$, define
\begin{equation}
s_t^{(i,k)}
=
\bar m_{t|T}^{(-t)}+\mathbf H^{(o)}_{t,t}o_t^{(k)}+L_t\xi^{(i)},
\qquad i=1,\dots,N.
\label{eq:mc_samples_scalar}
\end{equation}
Then
\begin{equation}
\widehat{\mu}_g^{(k)}
=
\frac1N\sum_{i=1}^N g(s_t^{(i,k)}),
\label{eq:mc_mu_scalar}
\end{equation}
and
\begin{equation}
\widehat{\nabla J}^{(k)}
=
2\bigl(\widehat{\mu}_g^{(k)}-M\bigr)
(\mathbf H^{(o)}_{t,t})^\top
\left(
\frac1N\sum_{i=1}^N \nabla g(s_t^{(i,k)})
\right).
\label{eq:mc_grad_scalar}
\end{equation}
The attacked observation is then updated through Projected Gradient Descent (PGD)  onto the same leave-one-out likelihood ellipsoid as in Lemma~\ref{lem:loo_pred_obs_past_future}
\begin{equation}
o_t^{(k+1)}
=
\Pi_{\mathcal X_\epsilon}\!\left(
o_t^{(k)}-\eta\,\widehat{\nabla J}^{(k)}
\right),
\qquad
\mathcal X_\epsilon
=
\left\{
o\in\mathbb R^{d_o}:
(o-\hat o_{-t})^\top S_{-t}^{-1}(o-\hat o_{-t})\le \rho_\epsilon
\right\}.
\label{eq:pgd_scalar}
\end{equation}
where $\Pi_{\mathcal X_\epsilon}$ denotes the Euclidean projection onto the feasible set $\mathcal X_\epsilon$, that is, it maps any point to the closest point in $\mathcal X_\epsilon$ in Euclidean distance. The approach is described in \autoref{alg:single_point_attack}.

\begin{algorithm}[ht]
\caption{Likelihood-constrained single-observation attack on a scalar smoothed target}
\label{alg:single_point_attack_on_g_scalar}
\begin{algorithmic}[1]
\Statex \textbf{Required:} SSM parameters $\{\mathbf A_u,\mathbf B_u,\mathbf F_u,\mathbf G_u,\mathbf W_u,\mathbf V_u\}_{u=1}^T$ and initial beliefs $(m_0,P_0)$.
\Statex \textbf{Input:} observations $o_{1:T}$, actions $a_{1:T-1}$, attack index $t$, constraint level $\epsilon$ (or radius $\rho_\epsilon$), scalar target $M\in\mathbb R$, scalar function $g:\mathbb R^{d_s}\to\mathbb R$, optional gradient $\nabla g$, step size $\eta$, number of iterations $K$, Monte Carlo size $N$.
\Statex \textbf{Output:} adversarial observation $o_t^{\mathrm{adv}}$.
\Statex \rule{\linewidth}{0.4pt}

\State Compute $(\hat o_{-t},S_{-t})$ and $\rho_\epsilon$ from Lemma~\ref{lem:loo_pred_obs_past_future}.
\State Compute $\mathbf H^{(o)}_{t,t}$ and the affine decomposition
$
m_{t|T}(o_t^{\mathrm{adv}})
=
\bar m_{t|T}^{(-t)}+\mathbf H^{(o)}_{t,t}o_t^{\mathrm{adv}}
$
from Lemma~\ref{lem:explicit_o_a_rewrite}.
\State Compute the smoothed covariance $P_{t|T}$ and a factor $L_t$ such that $L_tL_t^\top=P_{t|T}$.
\State Define the feasible set
$
\mathcal X_\epsilon=
\{o:(o-\hat o_{-t})^\top S_{-t}^{-1}(o-\hat o_{-t})\le \rho_\epsilon\}
$.
\State Initialize $o_t^{(0)}=\Pi_{\mathcal X_\epsilon}(o_t)$.
\State Sample common Gaussian noises $\xi^{(1)},\dots,\xi^{(N)}\sim\mathcal N(0,\mathbf I)$.
\For{$k=0,\dots,K-1$}
    \State For each $i=1,\dots,N$, form
    $
    s_t^{(i,k)}
    =
    \bar m_{t|T}^{(-t)}+\mathbf H^{(o)}_{t,t}o_t^{(k)}+L_t\xi^{(i)}
    $.
    \State Estimate
    $
    \widehat{\mu}_g^{(k)}
    =
    \frac1N\sum_{i=1}^N g(s_t^{(i,k)})
    $.
    \State Estimate
    $
    \widehat{v}^{(k)}
    =
    \frac1N\sum_{i=1}^N \nabla g(s_t^{(i,k)})
    $
    if analytic gradient is available; otherwise by finite-differences.
    \State Compute gradient estimate
    $
    \widehat{\nabla J}^{(k)}
    =
    2(\widehat{\mu}_g^{(k)}-M)(\mathbf H^{(o)}_{t,t})^\top \widehat{v}^{(k)}
    $.
    \State Take projected step
    $
    o_t^{(k+1)}
    =
    \Pi_{\mathcal X_\epsilon}\!\left(
    o_t^{(k)}-\eta\,\widehat{\nabla J}^{(k)}
    \right)
    $.
\EndFor
\State Return $o_t^{\mathrm{adv}}=o_t^{(K)}$.
\end{algorithmic}
\end{algorithm}

\subsection{\textbf{Adversarial attack in RL settings}}
\label{subsec:rl_attack}

We now specialize the attack construction to a RL setting.
The attacker is no longer interested in directly distorting the latent-state
estimate. Instead, its objective is to manipulate the agent's perception of the
current state so that the policy selects an action that is suboptimal for the
actual underlying state of the system. We preserve the plausibility constraint
introduced in \eqref{eq:attack_g_scalar}; only the attack objective changes.

Focus on the final decision time \(T\). Attacking the final observation is
particularly natural in this setting: the adversarial perturbation modifies the
agent's belief about the current state \(s_T\), which in turn changes the action
\(a_T\) selected by the policy. Importantly, the selected action is
executed in the actual latent state \(s_T^0\). Therefore, the attack should not merely induce an arbitrary state estimate, but
an action that performs poorly when evaluated at the actual state. This attack purpose is formalized in \eqref{eq:attack-RL}: modifying the agent's
perceived state so that the induced action is harmful for the actual state.
However, the attacker does not directly observe \(s_T^0\). It only has access to the posterior estimate $m_T$.

A computable approximation is obtained by replacing the unknown state \(s_T^0\). This gives the surrogate objective
\begin{equation}
\label{eq:approx_attack_objective_rl}
\widetilde{\mathcal J}_T
\left(
o_T^{\mathrm{adv}}
\right)
=
\mathbb E_{
s_T^{\mathrm{adv}}
\sim
q_T^{\mathrm{adv}}
(
\cdot;
o_T^{\mathrm{adv}}
)
}
\left[
\mathbb E_{
a_T
\sim
\pi(
\cdot
\mid
s_T^{\mathrm{adv}}
)
}
\left[
Q^\pi
\left(
m_T,
a_T
\right)
\right]
\right].
\end{equation}
For a fixed realization of the unknown state \(s_T^0\), define the resulting
approximation error as
\begin{equation}
\label{eq:true_state_substitution_error}
\mathcal E_T
\left(
o_T^{\mathrm{adv}};
s_T^0
\right)
=
\mathcal J_T
\left(
o_T^{\mathrm{adv}};
s_T^0
\right)
-
\widetilde{\mathcal J}_T
\left(
o_T^{\mathrm{adv}}
\right).
\end{equation}
Equivalently,
\begin{align}
\mathcal E_T
\left(
o_T^{\mathrm{adv}};
s_T^0
\right)
={}&
\mathbb E_{
s_T^{\mathrm{adv}}
\sim
q_T^{\mathrm{adv}}
(
\cdot;
o_T^{\mathrm{adv}}
)
}
\Bigg[
\mathbb E_{
a_T
\sim
\pi(
\cdot
\mid
s_T^{\mathrm{adv}}
)
}
\Big[
Q^\pi
\left(
s_T^0,
a_T
\right) -
Q^\pi
\left(
m_T,
a_T
\right)
\Big]
\Bigg].
\label{eq:true_state_substitution_error_expanded}
\end{align}

The ideal objective therefore admits the decomposition
\begin{equation}
\label{eq:true_objective_error_decomposition}
\mathcal J_T
\left(
o_T^{\mathrm{adv}};
s_T^0
\right)
=
\widetilde{\mathcal J}_T
\left(
o_T^{\mathrm{adv}}
\right)
+
\mathcal E_T
\left(
o_T^{\mathrm{adv}};
s_T^0
\right).
\end{equation}
which is not computable by the attacker as the correction
term still depends on the unknown state \(s_T^0\). Therefore, the purpose of the following theorem is to quantify the expected error,
\begin{equation}
\label{eq:expected_state_substitution_error}
\overline{\mathcal E}_T
\left(
o_T^{\mathrm{adv}}
\right)
=
\mathbb E_{
s_T^0
\sim
p(
\cdot
\mid
o_{1:T},
a_{1:T-1}
)
}
\left[
\mathcal E_T
\left(
o_T^{\mathrm{adv}};
s_T^0
\right)
\right].
\end{equation}

\begin{theorem}[Expected approximation error]
\label{thm:state_estimation_substitution_error}

Assume that \(Q^\pi(\cdot,a)\) is twice  differentiable with
respect to the state and that there exists \(M_{\max}<\infty\) such that
\begin{equation}
\label{eq:q_hessian_uniform_bound}
\left\|
\nabla_{ss}^{2}
Q^\pi(s,a)
\right\|_{\mathrm{op}}
\leq
M_{\max}
\end{equation}
for every relevant state \(s\) and action \(a\). Then, for every
adversarial observation \(o_T^{\mathrm{adv}}\),
\begin{equation}
\label{eq:total_rl_error_bound_expanded}
\left|
\overline{\mathcal E}_T
\left(
o_T^{\mathrm{adv}}
\right)
\right|
\leq
\frac{M_{\max}}{2}
\operatorname{tr}
\left[
\left(
\left(
\mathbf A_T
P_{T-1}
\mathbf A_T^\top
+
\mathbf W_T
\right)^{-1}
+
\mathbf F_T^\top
\mathbf V_T^{-1}
\mathbf F_T
\right)^{-1}
\right].
\end{equation}
\lemmaend
\end{theorem}

Theorem~\ref{thm:state_estimation_substitution_error} shows that the expected
error introduced by replacing the unknown state \(s_T^0\) with its posterior
mean \(m_T\) is controlled by two quantities: the posterior state uncertainty,
measured by \(\operatorname{tr}(P_T)\), and the curvature of \(Q^\pi\) with
respect to the state. Therefore, evaluating the induced action at \(m_T\) is
well justified whenever the state estimate is sufficiently accurate or
\(Q^\pi(\cdot,a)\) is approximately linear over the constraint region.

In this way, the approximation improves as the posterior covariance \(P_T\)
decreases which depends both on the transition and the observation covariance. Conversely, the effectiveness of the attack also depends on the
observation noise \(\mathbf V_T\). Indeed, the Kalman update can be
written as
\begin{equation}
m_T=
\left(
\mathbf I-K_T\mathbf F_T
\right)
\mathbf A_Tm_{T-1}
+
K_To_T
+
\left[
\left(
\mathbf I-K_T\mathbf F_T
\right)
\mathbf B_T
-
K_T\mathbf G_T
\right]
a_{T-1}.
\label{eq:kalman_weighted_rl_attack}
\end{equation}
Thus, a larger Kalman gain assigns greater weight to the adversarial observation and less weight to the model prediction. This Kalman gain,
\begin{equation}
K_T
=
\left(
\mathbf A_TP_{T-1}\mathbf A_T^\top+\mathbf W_T
\right)
\mathbf F_T^\top
\left[
\mathbf F_T
\left(
\mathbf A_TP_{T-1}\mathbf A_T^\top+\mathbf W_T
\right)
\mathbf F_T^\top
+
\mathbf V_T
\right]^{-1},
\label{eq:kalman_gain_general_rl}
\end{equation} shows that a smaller observation covariance \(\mathbf V_T\),
or a larger transition covariance $\mathbf W_T$ increases the relative influence of the current observation on the posterior mean.

Consequently, for a fixed posterior covariance \(P_T\), and therefore for a
fixed expected approximation-error bound, the attack becomes more effective when the transition covariance \(\mathbf W_T\) is larger relative to the
observation covariance \(\mathbf V_T\). This corresponds to a regime in which
a larger proportion of the prior uncertainty is generated by the transition
model rather than by the observation model. In this case, the defender greater weight to the incoming observation, allowing an adversarial
perturbation to induce a larger displacement of the posterior mean without
increasing the approximation-error bound.

Therefore, the most favorable regime for the attacker is characterized by
small observation noise and large transition  uncertainty. A small
\(\mathbf V_T\) keeps the clean posterior covariance \(P_T\) small, making
\(m_T\) an accurate estimate of the actual state \(s_T^0\), while a large
predictive covariance, which may arise from a large transition covariance
\(\mathbf W_T\), increases the Kalman gain and gives the manipulated
observation sufficient influence to substantially alter the agent's estimated
state.

Motivated by this approximation, the RL adversarial attack is defined
by the followng constrained optimization problem,
\begin{equation}
\boxed{
\begin{aligned}
\min_{o_T^{\mathrm{adv}}\in\mathbb R^{d_o}}
\quad &
\mathbb E_{
s_T^{\mathrm{adv}}
\sim
p_T^{\mathrm{adv}}
(
\cdot|
o_T^{\mathrm{adv}})
}
\left[
\mathbb E_{
a_T
\sim
\pi
}
\left[
\mathbb E_{s_{t+1}\sim p(\cdot\mid m_T,a_T)}
\left[
r(m_T,a_T)
+
\gamma V^\pi(s_{T+1})
\right]
\right]
\right]
\\[1ex]
\text{s.t.}
\quad &
\left(
o_T^{\mathrm{adv}}
-
\hat o_{-T}
\right)^\top
S_{-T}^{-1}
\left(
o_T^{\mathrm{adv}}
-
\hat o_{-T}
\right)
\leq
\rho_\epsilon .
\end{aligned}
}
\label{eq:practical_rl_attack_objective}
\end{equation}

In terms of the notation introduced in
\autoref{subsec:attack_target_function_estimation}, the scalar RL function
\(g(s_T):=g_{RL}(s_T)\) is now itself a double expectation. Although this may appear
computationally expensive, its evaluation requires only a one-step-ahead
prediction rather than simulating a complete trajectory.
Consequently, both the function and its gradient can be estimated using a
small number of MC samples.

The evaluations of \(g_{\mathrm{RL}}\) can be obtained using the
procedure in \autoref{alg:rl_target_function_gradient} which is to be included within  Lines 8--20 of \autoref{alg:single_point_attack_on_g_scalar}.

\begin{algorithm}[ht]
\caption{Computation of \(g_{RL}(s_T)\) and \(\nabla g_{RL}(s_T)\)}
\label{alg:rl_target_function_gradient}
\begin{algorithmic}[1]

\Statex \textbf{Required:} differentiable policy \(a_T=\pi(\cdot |s)\), reward function \(r\); value function \(V^\pi\); discount factor \(\gamma\).

\Statex \textbf{Input:} perceived state \(s_T\), number of policy samples
\(N_a\), and number of transition samples \(N_p\).

\Statex \textbf{Output:} 
\(\widehat g_{\mathrm{RL}}(s)\) and
\(\widehat{\nabla_s g_{\mathrm{RL}}}(s)\) estimates.

\Statex \rule{\linewidth}{0.4pt}

\State Initialize \(\widehat g_{\mathrm{RL}}(s)\gets 0\) and
\(\widehat{\nabla_s g_{\mathrm{RL}}}(s)\gets \mathbf 0\).

\For{\(j=1,\ldots,N_a\)}
    \State Sample \(a_T^{(j)}\sim\pi\).

    \For{\(\ell=1,\ldots,N_p\)}
        \State Sample \(\omega_{T+1}^{(j,\ell)} \sim \mathcal N(\mathbf 0,\mathbf W_{T+1})\) and compute \( s_{T+1}^{(j,\ell)} = \mathbf A_{T+1}m_T + \mathbf B_{T+1}a_T^{(j)} + \omega_{T+1}^{(j,\ell)} \).

        \State Compute the one-step return
        \(R^{(j,\ell)}
        =r(m_T,a_T^{(j)})
        +\gamma V^\pi(s_{T+1}^{(j,\ell)})\).

        \State Update
        \(\widehat g_{\mathrm{RL}}(s)
        \gets
        \widehat g_{\mathrm{RL}}(s)+\frac{1}{N_aN_p}R^{(j,\ell)}\).

        \State Compute         \(
        \nabla_s R^{(j,\ell)}
        \) by automatic differentation

        \State Update
        \(
        \widehat{\nabla_s g_{\mathrm{RL}}}(s)
        \gets
        \widehat{\nabla_s g_{\mathrm{RL}}}(s)
        +
        \frac{1}{N_aN_p}
        \nabla_s R^{(j,\ell)}.
        \)
    \EndFor
\EndFor

\State Return
\(\widehat g_{\mathrm{RL}}(s)\) and
\(\widehat{\nabla_s g_{\mathrm{RL}}}(s)\).

\end{algorithmic}
\end{algorithm}

\FloatBarrier

\section{Online Bayesian defense: Directional Covariance Adaptation}
\label{sec:OnlineBayesianDefense}
This section proposes an online Bayesian mechanism to adapt the observation noise covariance in presence of adversarial perturbations or, more generally, along directions in which the observation noise may be inherently more disruptive. The perspective is consistent with generalized Bayesian online learning with auxiliary latent variables, as formalized in the BONE framework \citep{duran2024unifying}, also related to recent robust Kalman filtering methods that modify the measurement update \citep{duranmartin2024outlier}.

Let $\hat{o}_t$ denote the one-step-ahead predictive observation mean based on \eqref{eq:kf_pred_obs}, $o_t$ the actual observation, and $o_t^{\mathrm{adv}}$ be the adversarial observation generated according to the attack criterion described in \autoref{sec:problemstatement} and \autoref{sec:Methodology}. Define the corresponding normalized observation direction as
\begin{equation}
u_t
=
\frac{o_t-\hat{o}_t}{\|o_t-\hat{o}_t\|}.
\label{eq:adv_direction}
\end{equation}

To distinguish between nominal and contaminated measurements, we introduce a latent Bernoulli variable
\begin{equation}
z_t \sim \mathrm{Bernoulli}(\pi_t),
\label{eq:latent_bernoulli}
\end{equation}
where $z_t=1$ indicates adversarial contamination and $z_t=0$ corresponds to the nominal regime. The prior contamination probability $\pi_t$ should encode, before assimilating the current observation, how plausible it is that the next measurement may be adversarially corrupted. In the language of \citet{duran2024unifying}, this quantity plays the role of a predictive gating weight for the auxiliary contamination process. We construct this prior from two complementary experts.

The first is a \textit{state-risk expert}. It measures how dangerous or
vulnerable the predicted state is for the task under consideration. In our
setting, it is based on the attacked objective $g(s_t)$ that we seek to
protect. As an example, define
\begin{equation}
r_t^{(s)}
=
\mathbb{P}\!\left(
g(s_t)\geq c
\mid
o_{1:t-1}
\right),
\label{eq:risk_expert_s}
\end{equation}
where $c$ is a prescribed danger threshold. This term is large when the
predictive distribution of the state places substantial mass in a dangerous
region\footnote{Here, the dangerous region denotes a subset of the state
space in which a specific decision should be taken.}. If such a threshold
$c$ is not known, vulnerability may instead be quantified through the local
sensitivity of $g$ around the predictive state. The idea is that uncertainty
becomes more critical in regions where small perturbations induced by
observation noise can produce large changes in the protected objective. A
natural alternative is therefore to define a gradient-based expert of the
form
\begin{equation}
r_t^{(s)}
=
\sigma\!\left(
\alpha+\beta\left\|\nabla g(m_{t|t-1})\right\|
\right),
\label{eq:risk_expert_grad}
\end{equation}
where $\sigma(\cdot)$ is the logistic function and $\alpha,\beta$ are tuning
parameters. This quantity approaches $1$ when the objective is highly
sensitive to local state deviations, that is, when observation noise can have
a stronger effect on the quantity we wish to protect, even when no explicit
danger threshold is available.

The second one is a \textit{fixed-caution expert}, which assigns a constant
risk level determined only by the constraint level \(\epsilon\). Define
\begin{equation}
r_t^{(c)}
=
\epsilon,
\label{eq:fixed_caution_score}
\end{equation}
Therefore, this is a minimum level of risk that the defender wants to have, independently of the environment we are working.

We combine both prior experts through a weighted sum,
\begin{equation}
\pi_t
=
\omega_s r_t^{(s)}
+
\omega_c r_t^{(c)},
\qquad
\omega_s,\omega_c\geq0,
\qquad
\omega_s+\omega_c=1.
\label{eq:prior_pi_general_weights}
\end{equation}
This prior combines task-level vulnerability with a desired level of defense. The weights must be selected according to the desired behavior.
Assigning greater weight to the second component introduces a persistent baseline level of caution that does not depend on the $g$ objective. In contrast, assigning greater weight to the first one makes the prior more task-specific. Thus, $\omega_c$ controls a more conservative form of caution, whereas $\omega_s$ produces a more
decision-dependent assessment of risk.

Conditioned on \(z_t\), we modify the observation-noise model as
\begin{equation}
\nu_t\mid z_t
\sim
\mathcal N
\left(
\mathbf 0,\,
\mathbf V_t+\lambda z_tu_tu_t^\top
\right),
\label{eq:robust_obs_noise}
\end{equation}
which implies
\begin{equation}
o_t\mid s_t,a_{t-1},z_t
\sim
\mathcal N
\left(
\mathbf F_ts_t+\mathbf G_ta_{t-1},\,
\mathbf V_t+\lambda z_tu_tu_t^\top
\right).
\label{eq:robust_obs_model}
\end{equation}
Here, \(\lambda>0\) controls the magnitude of the covariance inflation.
Therefore, when \(z_t=1\), the observation uncertainty is increased only along
the direction \(u_t\), while directions orthogonal to \(u_t\) remain
unchanged. This selective covariance adaptation is related in spirit to
variational Bayesian noise-adaptation methods, which learn time-varying
observation covariances online
\citep{sarkka2009recursive,sarkka2013variational}, but here the correction is
restricted to a structured rank-one perturbation.

To update the prior into a posterior disturbance probability, we use two complementary
diagnostics of adversarial behavior and incorporate the evidence through a
generalized Bayesian update.

The first expert, \textit{distance-aware expert}, measures the magnitude of the observed innovation \(r_t
=
o_t-\hat o_t.
\) through its squared Mahalanobis distance, \(
d_t^2
=
r_t^\top
\mathbf S_t^{-1}
r_t,\)
and the corresponding evidence score
\begin{equation}
a_{\mathrm{md},t}
=
\frac{
d_t^2
}{
\epsilon_{\mathrm{md}}+d_t^2
},
\label{eq:mahalanobis_attack_evidence}
\end{equation}
where \(\epsilon_{\mathrm{md}}>0\) controls the sensitivity of the score.
Thus, \(a_{\mathrm{md},t}\) is close to zero for small innovations and
approaches one as the observation becomes increasingly anomalous under the
nominal predictive model.

The second one, \textit{function-aware expert}, measures how the inferred state degrades the objective. Define
\begin{equation}
a_{\mathrm{obj},t}
=
\operatorname{clip}
\left(
\frac{
J(o_t)-J(\hat{o_t})
}{
J(o_t^{\mathrm{adv}})-J(\hat{o_t})
},
0,
1
\right),
\label{eq:objective_gap_attack_evidence}
\end{equation}
where
\(\operatorname{clip}(x,0,1)
=
\min
\left\{
1,
\max\{0,x\}
\right\}.
\). Hence, \(a_{\mathrm{obj},t}=0\) indicates behaviour similar to the predictive
case, whereas \(a_{\mathrm{obj},t}=1\) indicates a degradation at least as
large as the one induced by the computed adversarial attack.

The two evidence scores are aggregated through a weighted linear combination,
\begin{equation}
a_t
=
w_{\mathrm{md}}a_{\mathrm{md},t}
+
w_{\mathrm{obj}}a_{\mathrm{obj},t},
\qquad
w_{\mathrm{md}}+w_{\mathrm{obj}}=1,
\label{eq:combined_attack_evidence}
\end{equation}
where again, giving more weight to $w_{\mathrm{obj}}
$ builds a system more worried in the decision taken based on $g$, while giving more weight to $w_{\mathrm{md}}$ builds a system which is more focused on identifying outliers independently on how they affect the decision.

The aggregated score is incorporated through a generalized Bayesian update.
For \(z_t\in\{0,1\}\), the binary logarithmic loss is 
\(\ell_t(z_t)
=
-z_t\log a_t
-
(1-z_t)\log(1-a_t).\)
and the associated generalized likelihood is
\begin{equation}
\widetilde L_t(z_t)
=
\exp
\left(
-\ell_t(z_t)
\right)
=
a_t^{z_t}
(1-a_t)^{1-z_t}.
\label{eq:aggregated_generalized_likelihood}
\end{equation}
Consequently, the generalized posterior probability is
\begin{equation}
\gamma_t
=
\widetilde{\mathbb P}
\left(
z_t=1
\mid
o_{1:t}
\right)
=
\frac{
\pi_t a_t
}{
\pi_t a_t
+
(1-\pi_t)(1-a_t)
}.
\label{eq:posterior_gamma}
\end{equation}

Finally, we use the generalized posterior distribution of \(z_t\) to average
the conditional observation covariance:
\begin{align}
\widetilde{\mathbf V}_t
=
\widetilde{\mathbb E}
\left[
\mathbf V_t
+
\lambda z_tu_tu_t^\top
\;\middle|\;o_t
\right]=
\mathbf V_t
+
\lambda
\widetilde{\mathbb E}
\left[
z_t
\mid o_t
\right]
u_tu_t^\top=
\mathbf V_t
+
\lambda\gamma_tu_tu_t^\top.
\label{eq:posterior_averaged_covariance}
\end{align}
therefore, the observation covariance becomes
\(\widetilde{\mathbf V}_t
=
\mathbf V_t
+
\lambda\bar\gamma_tu_tu_t^\top.
\)
Optionally, we threshold the generalized posterior probability according to
\(\bar\gamma_t
=
\gamma_t
\mathbf 1
\left\{
\gamma_t\geq\delta
\right\}.\)
If \(\gamma_t<\delta\), the available evidence is considered insufficient and
the covariance correction is switched off.

Therefore, the generalized posterior probability controls the
magnitude of the directional covariance inflation continuously. When
\(\gamma_t\) is close to zero, the nominal observation covariance is retained;
when \(\gamma_t\) is close to one, the full rank-one covariance inflation is
applied.

The effect of the covariance inflation can be quantified by comparing the
posterior updates with and without the defence. Since
\(
o_t-\hat o_t
=
\|o_t-\hat o_t\|u_t,
\)
the additional covariance is aligned with the observed innovation. Using the
Sherman--Morrison formula, the difference between the corresponding posterior
means is
\begin{equation}
m_{t|t}^{z_t=1}
-
m_{t|t}^{z_t=0}
=
-
\frac{\alpha_t-1}{\alpha_t}
\mathbf P_{t|t-1}\mathbf F_t^\top
\mathbf S_t^{-1}
\left(o_t-\hat o_t\right),
\label{eq:posterior_mean_difference}
\end{equation}
where
\(\alpha_t
=
1+
\lambda\bar\gamma_t
u_t^\top\mathbf S_t^{-1}u_t.
\) Since \(\alpha_t\geq 1\), the observation has less influence on the current hidden-state estimate whenever the defense is activated.

At the same time, the observation covariance is modified only through the
rank-one term
\(
\lambda\bar\gamma_tu_tu_t^\top.
\) Hence, additional uncertainty is introduced only along the suspicious
observation direction \(u_t\); directions orthogonal to \(u_t\) receive no
direct covariance inflation. Therefore, the defense reduces trust in the
component of the observation associated with the detected innovation while
preserving the information carried by the remaining observation directions.

The corresponding increase in posterior state uncertainty is
\begin{equation}
\mathbf P_{t|t}^{z_t=1}
-
\mathbf P_{t|t}^{z_t=0}
=
\frac{\lambda\bar\gamma_t}{\alpha_t}
v_tv_t^\top
\succeq 0,
\qquad \text{with} \qquad
v_t
=
\mathbf P_{t|t-1}\mathbf F_t^\top
\mathbf S_t^{-1}u_t,
\label{eq:posterior_covariance_difference}
\end{equation}
so the additional posterior uncertainty is also rank one and concentrated
along the state-space direction induced by \(u_t\). This produces the desired twofold effect. First, the current observation has a reduced influence on the hidden-state estimate along the suspicious direction, rather than being discarded entirely. Second, the resulting posterior covariance records that the state is less certain along the corresponding state-space direction. After propagation through the dynamics, this increased uncertainty is carried into the next predictive covariance. Consequently, if the subsequent observation \(o_{t+1}\) provides information along that direction, the Kalman filter places relatively more weight on this new measurement because the information inherited from the past is less reliable. At the same time, directions orthogonal to the inflated direction remain unchanged and retain the same level of confidence as under the nominal filter.

\FloatBarrier

\section{Experiments}
\label{sec:experiments}

The experiments will be presented in the same order as there were introduced in \autoref{sec:Methodology}. Each attack will be accompanied by the proposed covariance adaptation defense results. Moreover, to maintain a common interpretation across experiments, all examples can be viewed through an scenario in which a robot receiving information from the environment through different sensors, summarized in $o_t$. However, to encompass all possible cases, $o_t$ may have different interpretations and serve different purposes.

\subsection{\textbf{Adversarial attacks on hidden-state estimation} (\autoref{subsec:lgssm_attack}) }
\label{subsec:attack_state_estimation}

\begin{figure}[t]
    \centering
    \includegraphics[width=0.99\linewidth]{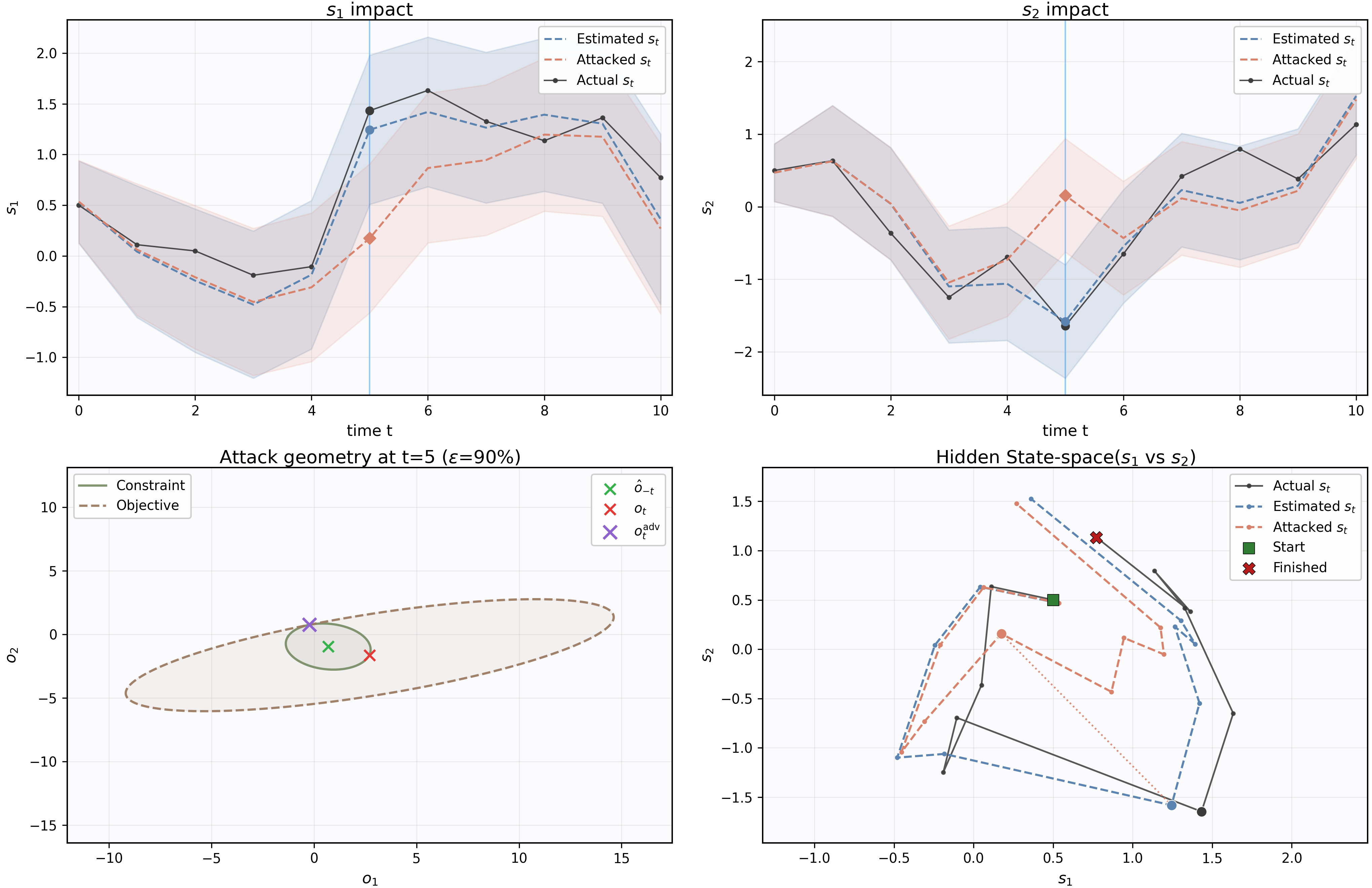}
    \caption{Impact of adversarial perturbation at time $t=5$.
    \textbf{Top-right:} Attack geometry in observation space. 
    The solid ellipse denotes the likelihood constraint. Dashed ellipse represents the objective level set evaluated at the optimum. \textbf{Top-left and bottom-left:} Smoothed hidden-state components $s_1$ and $s_2$ under baseline and adversarial settings, including $95\%$ credible intervals.
    \textbf{Bottom-right:} State-space trajectory comparison ($s_1$ versus $s_2$).}
    \label{fig:Results/AdvNaiveAttack.png}
\end{figure}

This case involves a two-dimensional input $o_t$. It represents the position of a navigation system operating on a robot. Consequently, in this setting, we focus exclusively on altering the robot's perception of its position. The model
matrices used are
{\small
\[
\mathbf A
=
\begin{bmatrix}
0.65 & 0.40\\
-0.15 & 0.70
\end{bmatrix},
\qquad
\mathbf B
=
\begin{bmatrix}
1.65 & 1.40\\
-0.15 & 0.70
\end{bmatrix},
\qquad
\mathbf F=\mathbf I_2,
\qquad
\mathbf G=\mathbf 0,
\]
\[
\mathbf W
=
\begin{bmatrix}
0.4800 & -0.2325\\
-0.2325 & 0.2100
\end{bmatrix},
\qquad
\mathbf V
=
\begin{bmatrix}
0.2730 & 0.0525\\
0.0525 & 0.7140
\end{bmatrix}.
\]
}
And the actions are independently sampled at each time step according to
\[
a_t\sim\operatorname{Unif}\!\left([-0.5,0.5]^2\right),
\qquad
t=0,\ldots,T-1.
\]

Results are shown in \autoref{fig:Results/AdvNaiveAttack.png}, where a
time-homogeneous SSM of length $T=12$ is attacked at time $t=5$. A first important conclusion is that perturbing a single observation does not typically affect the entire sequence. Instead, its impact is mainly concentrated on the targeted state and on nearby states, both before and after the attack time.

More importantly, the attack not only yields the worst-case perturbation under the prescribed constraint, but also identifies the most vulnerable direction in observation space. This information can guide improvements in sensing and filtering to reduce performance degradation and loss of accumulated reward.

\begin{figure}[t]
    \centering

    \subfloat[
        Dependence of attack direction on the constraint level
        $\epsilon$. The figure also shows the actual observation,
        the most likely observation, and the objective-function
        ellipses.
        \label{fig:epsilon-dependence}
    ]{
        \includegraphics[width=0.36\linewidth]{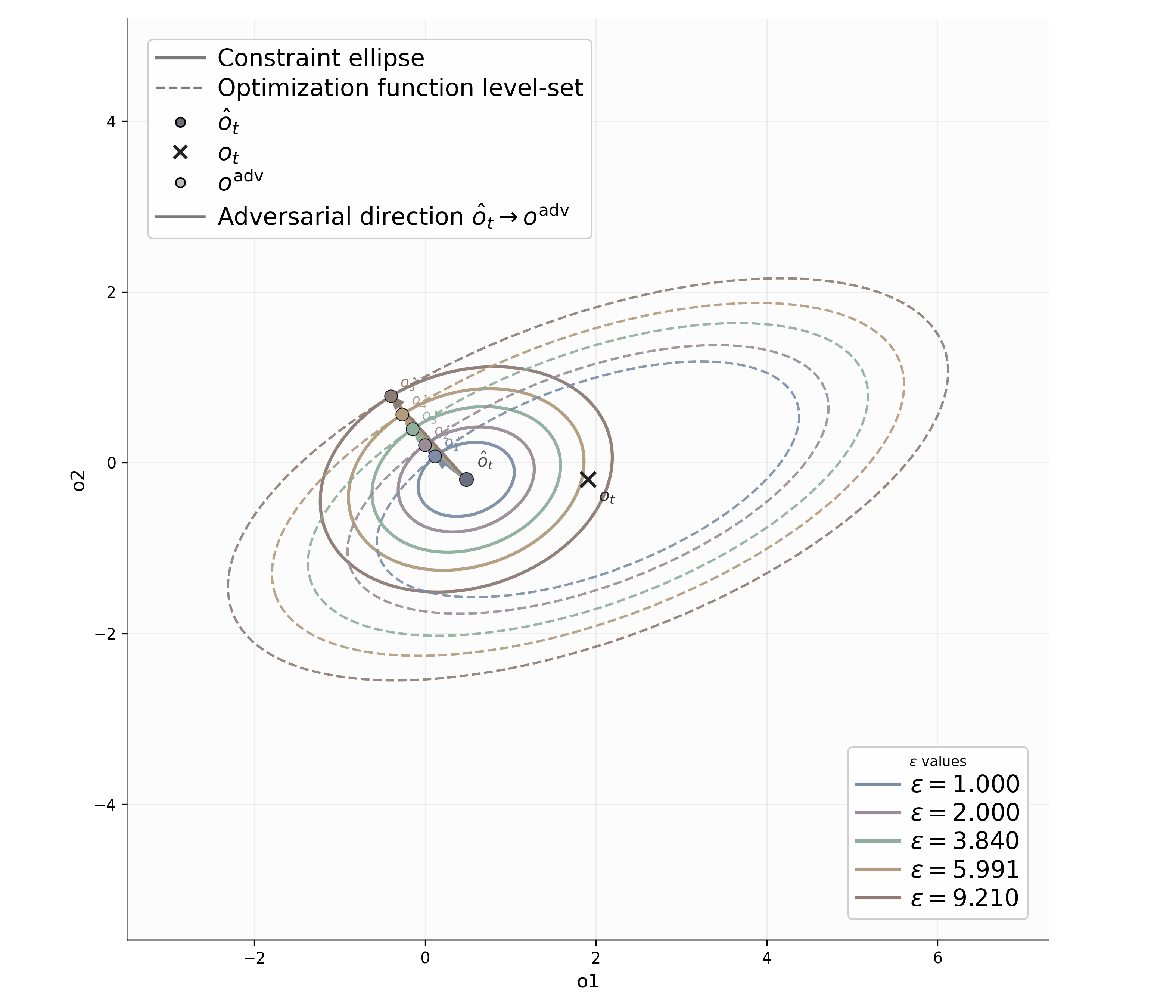}
    }
    \hspace{0.02\linewidth}
    \subfloat[
        Dependence of attack impact on attacked time step $t$.
        Top panel reports local estimation error at time $t$;
        bottom panel shows global error over the full state sequence.
        \label{fig:time-dependence}
    ]{
        \includegraphics[width=0.52\linewidth]{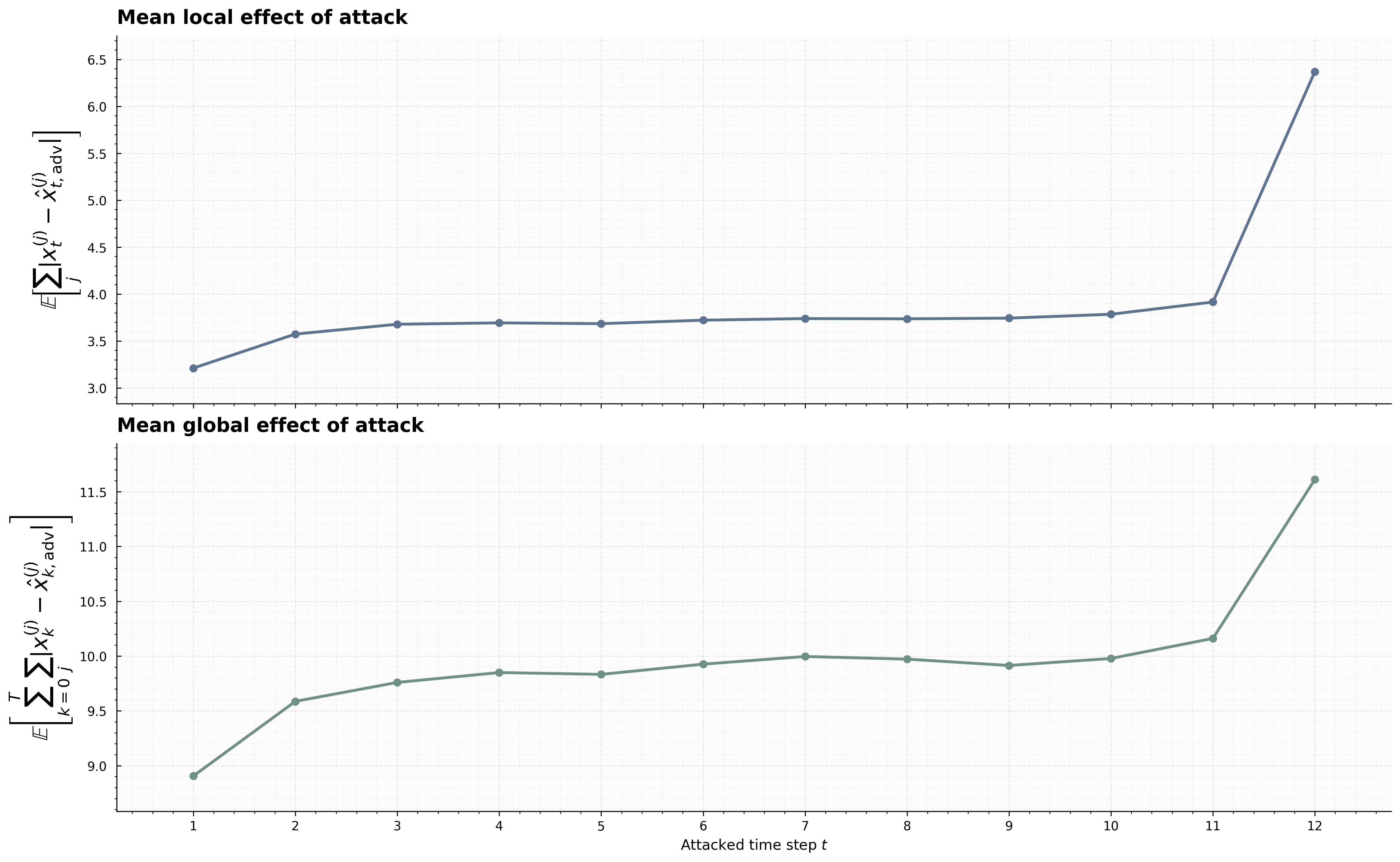}
    }

    \caption{Effect of constraint level $\epsilon$ and of attacked time
    step $t$ on adversarial perturbations and its impact on state
    estimation.}
    \label{fig:epsilon-time-dependence}
\end{figure}

A second relevant observation concerns the role of the constraint level $\epsilon$. Larger values enable stronger attacks, which generally produce larger estimation errors at the attacked time $t$. These stronger perturbations also tend to spread more visibly to neighboring state estimates. In addition, changing $\epsilon$ slightly  alter its perturb direction $o_t^{\text{adv}}-\hat o_t $ as \autoref{fig:epsilon-dependence} illustrates. Consequently, the most disruptive attack direction also depends on the likelihood constraint level.

Moreover, for a fixed constraint level $\epsilon$, the attack magnitude varies
with the attacked time index $t$. To examine this dependence, the entries of
$\mathbf A$, $\mathbf B$, $\mathbf F$, and $\mathbf G$, as well as those of
the raw matrices $\mathbf W_{\mathrm{raw}}$ and
$\mathbf V_{\mathrm{raw}}$, are sampled independently from
$\mathcal N(0,\sigma^2)$. Valid covariance matrices are then constructed as
$\mathbf W=\operatorname{proj}_{\mathbb S_+}
(\mathbf W_{\mathrm{raw}})+0.05\mathbf I$ and
$\mathbf V=\operatorname{proj}_{\mathbb S_+}
(\mathbf V_{\mathrm{raw}})+0.05\mathbf I$, where
$\operatorname{proj}_{\mathbb S_+}$ denotes projection onto the
positive-semidefinite cone. \autoref{fig:time-dependence} reports the local and global effects of the
perturbation across time. The local effect is the estimation error at the
attacked instant $t$, whereas the global effect measures the total error over
the full state sequence. The global distortion remains largely concentrated
around neighboring states, and both error criteria display a similar
qualitative pattern. Attacks at intermediate times produce comparable errors,
the final time step has a noticeably larger effect, and the initial and
penultimate time steps yield slightly smaller errors.

This behavior can be understood through the role of smoothing. For intermediate times, the effect of the perturbed observation is partially corrected by combining information from both past and future observations, which effectively reduces the feasible region under the likelihood constraint and limits the attack impact. At $t=T$, however, no future observations are available, so this corrective effect does not apply, making the attack more effective. At the other extreme, when $t=1$, all subsequent observations contribute to smoothing, leading to the strongest correction and therefore to the smallest error for a fixed constraint level $\epsilon$. 

\begin{figure}[t]
    \centering
    \includegraphics[width=0.99\linewidth]{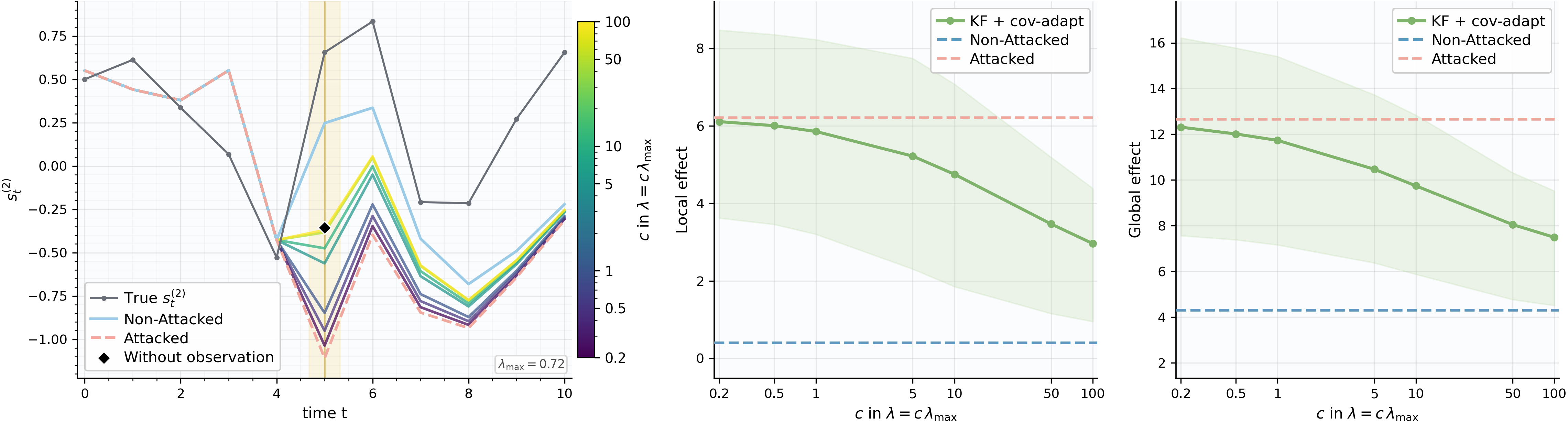}
    \caption{$\lambda$ dependence in the covariance adaptation defense mechanism
    \textbf{Left:} Second component of the hidden state space estimate for different lambda values. Black point corresponds to $\hat{s}_{t|-t}$. \textbf{Middle \& Right:} $\lambda$ dependence on the local/global effect of hidden state estimation averaged over $N=1000$ runs.}
    \label{fig:Results/lambda_sweep.png}
\end{figure}
    
Finally, the effect of the proposed defense mechanism introduced in \autoref{sec:OnlineBayesianDefense} is illustrated in \autoref{fig:Results/lambda_sweep.png} where the same randomly generated matrices where used. The leftmost panel shows how the second coordinate of the hidden-state estimate changes under an attack occurring at $t=5$ for different values of $\lambda$. To ensure generalization $\lambda$ is expressed as
\(
\lambda = c\lambda_{\max},
\)
where $\lambda_{\max}$ denotes the largest eigenvalue of the observation covariance matrix and $c$ is a scaling factor. We also set $\omega_s=\omega_c=0.5$ when constructing the prior attack probabilities, assigning equal weight to the state-risk and fixed-caution experts.

Results show that increasing $\lambda$ reduces both local and global effects of attack. Consequently, as $\lambda \to \infty$, the observation is effectively discarded, and the filtered hidden-state estimate converges to the prediction obtained without incorporating information at time $t$, namely $\hat{s}_{t|t-1}$. Additionally, the posterior probability of an attack increases only when the observation lies sufficiently close to the identified adversarial one. Therefore, for clean observations, $p_t^{(1)}(o_t) \approx 0$ and, consequently, $\gamma_t \approx 0$, so the covariance is not inflated and the resulting state inference remains essentially unchanged.

\FloatBarrier

\subsection{\textbf{Adversarial attacks on target-function estimation} (\autoref{subsec:lgssm_attack_on_g}) }
\label{subsec:attack_target_function_estimation}

\begin{figure}[t] \centering \includegraphics[width=0.995\linewidth]{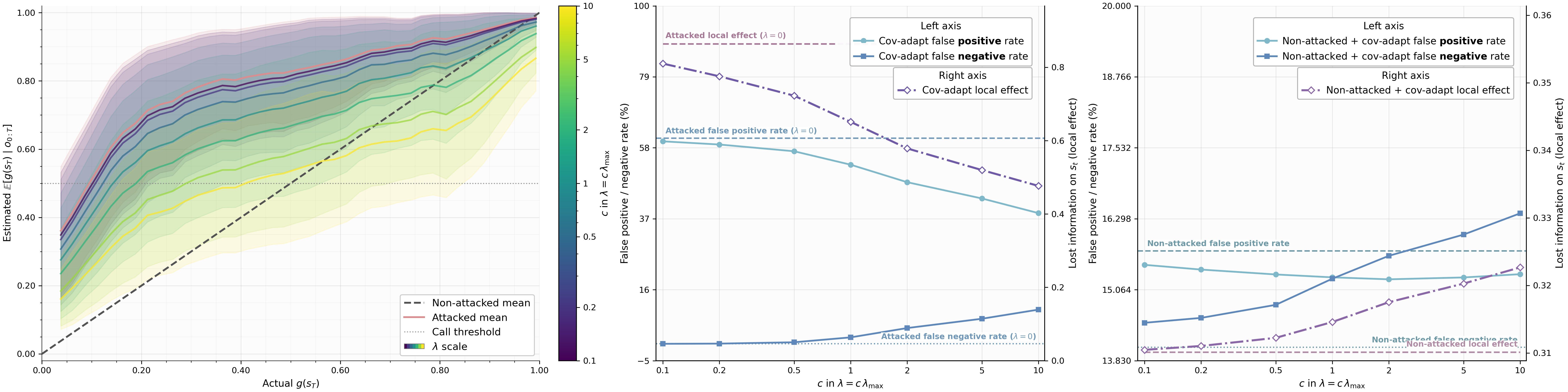} \caption{ Adversarial attacks on target-function estimation. \textbf{Left:} Comparison of $\mathbb{E}[g]$ under attack, both without defense (\textit{attacked}) and with directional covariance adaptation (\textit{DirCovAdapt}) for different $\lambda$ values. Dashed line represents non-attacked scenario. \textbf{Middle:} $\lambda$  dependence under attack with covariance adaptation. Left axis shows false-negative and false-positive rates, while right axis shows local effect, $\lvert \hat{s}^{\mathrm{adv}}_t-s_t \rvert^2$. Horizontal dashed lines represent values with no defense. \textbf{Right:} $\lambda$  dependence in the non-attacked setting with covariance adaptation. Left axis shows false-negative and false-positive rates, while Right axis shows local effect, $\lvert \hat{s}^{\mathrm{adv}}_t-s_t \rvert^2$. Horizontal dashed lines represent values without defense. } \label{fig:g_attack} \end{figure}

Following the robot-sensor example, we now extend the analysis to a three-dimensional setting, this a robot moving within then three dimensional space. The latent state \(s_t \in \mathbb{R}^3\) describes the unobserved condition of the system, the observation \(o_t \in \mathbb{R}^3\) contains the corresponding sensor measurements, and the control input \(u_t \in \mathbb{R}^3\) represents external interventions or operating conditions.

Unlike in \autoref{subsec:attack_state_estimation}, where the attack directly targets the state estimate \(s_t\), we now consider an attacker whose objective is to manipulate a scalar nonlinear functional of the latent state. In particular, we define the sigmoid-based score
\begin{equation} \begin{aligned} g(s_t)=\sigma\!\Big(&-3.2+1.0s_{t,1}+1.6s_{t,2}+1.6s_{t,3} +3.5s_{t,2}s_{t,3}\\ &+3.0s_{t,2}^{2}+3.0s_{t,3}^{2} +0.1s_{t,1}s_{t,2}+0.2s_{t,1}s_{t,3}\Big). \end{aligned} \label{eq:example_scalar_g} \end{equation} where \(\sigma(z)=1/(1+e^{-z})\). This score is interpreted as the probability of triggering an automated response, such as an alarm or a maintenance action. Therefore, the attacker aims to alter a nonlinear, decision-relevant quantity derived from the state.

For the numerical illustration, we set \(T=12\) and apply the attack at \(t=T\), which also allows us to evaluate the online covariance-adaptation defense. We adopt the decision rule
\(\mathbb{E}[g(s_t)] > 0.5\), and assume that the attacker seeks to increase \(\mathbb{E}[g(s_t)]\) as close to $1$ as possible, thereby maximizing the probability $g(s_t)$ that the automated response is triggered.

We perform several runs under the same attack configuration while applying online covariance adaptation with different values of \(\lambda\). Again, to obtain a scale-independent parametrization, we write
\(\lambda = c\lambda_{\max}\), where \(\lambda_{\max}\) denotes the largest eigenvalue of the observation covariance matrix, prior attack weights are $\omega_s=\omega_c=0.5$ and we also set a threshold for the posterior probability $\gamma_t$ of $0.3$. Results are reported in \autoref{fig:g_attack}.

The leftmost panel reports $\mathbb{E}[g(s_t)]$ under attack, both without defense and with covariance adaptation for different values of $\lambda$. As $\lambda$ increases, the defensive mechanism reduces the influence of the adversarial observation and brings the estimates closer to the identity line. However, excessively large values of $\lambda$ discount the observation too strongly, causing the defended estimate to fall below this line, especially for observations associated with $\mathbb{E}[g(s_t)]>0.5$, which are more likely to be interpreted as suspicious. Therefore, $\lambda$ must balance attack mitigation against excessive information loss.

The middle panel illustrates the same trade-off in the attacked setting. Increasing $\lambda$ reduces the false-positive rate, since the adversarial observation has less influence on the decision, and also decreases the local estimation error. Nevertheless, this improvement is accompanied by an increase in the false-negative rate. When the alarm should genuinely be triggered, the defense may incorrectly treat the corresponding observation as adversarial and discount it too strongly. This again suggests that an intermediate value of $\lambda$ is preferable.

\begin{figure}[t]
    \centering
    \includegraphics[width=0.6\linewidth]{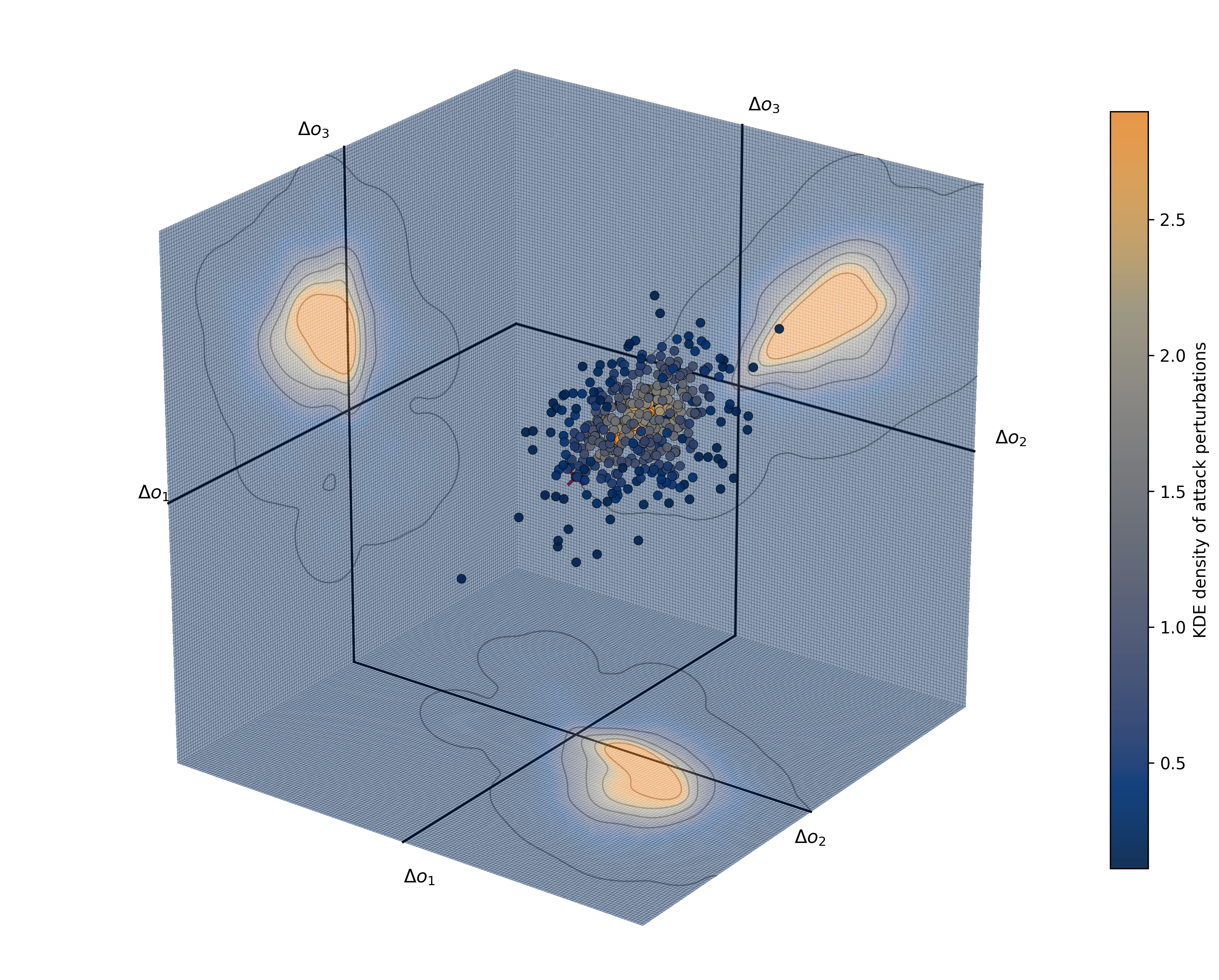}
    \caption{Estimated density of most disruptive directions in a 3D observation space for a fixed constraint level $\epsilon$.}
    \label{fig:attack_density}
\end{figure}

The rightmost panel considers the non-attacked setting and evaluates the effect of the defense in the presence of ordinary observational noise. A similar trade-off is observed: covariance adaptation reduces the false-positive rate but increases the false-negative rate. Moreover, as informative observations are progressively discounted, the local estimation error increases with $\lambda$. Therefore, the appropriate choice of $\lambda$ depends on the relative importance assigned by the defender to false positives and false negatives. Although the reduction in the false-positive rate is only around $1\%$, this improvement should not be regarded as negligible. In a sequential setting, where decisions are made at every time step, even a small increase in accuracy may accumulate and lead to a substantial improvement in overall performance, as \autoref{subsec:attack_rl_settings} shows.

Lastly, as an auxiliary tool, if the attack simulation is repeated multiple times, one can estimate a density over the most disruptive directions in the observation space. This is shown in Figure~\ref{fig:attack_density} where the first dimension does not seem to impact the disruptive direction. This procedure makes it possible to identify the least robust direction, that is, the direction along which perturbations most strongly exploit vulnerabilities not necessarily in the state estimation itself, but rather in the objective function of interest. In other words, these are the directions that have the greatest potential to alter downstream decision-making.

\FloatBarrier

\subsection{\textbf{Adversarial attacks in liner RL settings} (\autoref{subsec:rl_attack}) }
\label{subsec:attack_rl_settings}

To illustrate the RL attack setting, we aga in consider a two-dimensional navigation task with time-varying wind, formulated as a four-dimensional linear-Gaussian state-space model. The objective is to guide the agent toward a specified goal position while accounting for both its current position and the displacement induced by the wind. The latent state is
\[
s_t=
\begin{bmatrix}
p_{x,t} & p_{y,t} & d_{x,t} & d_{y,t}
\end{bmatrix}^{\top}
\in\mathbb{R}^4,
\]
where $(p_{x,t},p_{y,t})$ denotes the actual position of the agent and $(d_{x,t},d_{y,t})$ represents the current wind vector. The control input is the bounded two-dimensional action
\[
a_t=
\begin{bmatrix}
a_{x,t} & a_{y,t}
\end{bmatrix}^{\top}
\in[-1,1]^2,
\]
which determines the displacement selected by the agent toward a goal. The environment dynamics are described by
\begin{equation}
\label{eq:rl_wind_ssm_dynamics}
s_{t+1}=A_t s_t+B a_t+w_t,
\qquad
o_t=F s_t+v_t,
\end{equation}
where $w_t\sim\mathcal{N}(0,Q)$ and $v_t\sim\mathcal{N}(0,R)$ denote the process and observation noise, respectively. The transition, control, and observation matrices are
\[
A_t=
\begin{bmatrix}
1 & 0 & 1 & 0\\
0 & 1 & 0 & 1\\
0 & 0 & \rho_w\cos(\Delta\psi_t)
    & -\rho_w\sin(\Delta\psi_t)\\
0 & 0 & \rho_w\sin(\Delta\psi_t)
    & \rho_w\cos(\Delta\psi_t)
\end{bmatrix},
\qquad
B=
\begin{bmatrix}
I_2\\
0
\end{bmatrix},
\qquad
F=I_4.
\]
Thus, the position evolves according to the selected action and the current wind displacement, while the lower-right block of $A_t$ describes the wind evolution. The parameter $\rho_w\in(0,1)$ controls the persistence of the wind, whereas the rotation matrix $R(\Delta\psi_t)$ changes its direction by an angle $\Delta\psi_t$ at each time step.

The policy does not observe the observation directly. Instead, it receives the built hidden state
\[
o_t=
\begin{bmatrix}
y_{x,t} & y_{y,t} & y_{d_x,t} & y_{d_y,t}
\end{bmatrix}^{\top}
\in\mathbb{R}^4,
\]
which is transformed into the policy input
\begin{equation}
\label{eq:rl_wind_policy_obs}
z_t=
\begin{bmatrix}
(g_x-p_{x,t})/R_{\max}\,\,\,, 
&
(g_y-p_{y,t})/R_{\max}\,\,\,, 
&
d_{x,t}\,\,\,,
&
d_{y,t}
\end{bmatrix}^{\top}
\in\mathbb{R}^4,
\end{equation}
where $g=(g_x,g_y)$ is the goal location and $R_{\max}$ is the maximum navigation radius region, used to normalize the relative-position components and avoid scaling differences. Therefore, the first two coordinates encode the measured relative displacement to the goal, while the last two provide the measured wind vector.

An episode terminates when the actual position enters a ball of radius $r_{\mathrm{goal}}$ around the goal or when the maximum horizon is reached. Rewards is given as the agent gets closer to the goal with a penalty cost of $-1.0$ until termination, together with a positive terminal reward upon reaching the goal and a negative terminal penalty upon timeout.

The control policy is learned using Proximal Policy Optimization (PPO) \citep{schulman2017proximal} with an actor--critic architecture. The actor defines a two-dimensional squashed Gaussian policy,
\[
u_t\sim\mathcal{N}\!\left(\mu_\theta(z_t),\sigma^2I_2\right),
\qquad
a_t=\tanh(u_t)\in[-1,1]^2,
\]
where $\sigma$ is fixed, while the critic approximates the value function $V_\theta(z_t)$. During evaluation, the deterministic mean action
\[
a_t=\tanh\!\left(\mu_\theta(z_t)\right)
\]
is used.

The first step is to assess how well the approximate attack objective,
\(Q^\pi(m_T,a_T)\), performs relative to the ideal objective,
\(Q^\pi(s_T^0,a_T)\), based on the true state \(s_T^0\). We compare the
effectiveness of the attack based on the posterior mean \(m_T\) with that of an
oracle attack using \(s_T^0\), which is available only for evaluation. Moreover, the attack is not applied at every
time step. Instead, it is performed with probability
\(p_{\mathrm{attack}}=0.10\), ensuring that consecutive attacks are sufficiently
separated and do not influence each other.

Following \autoref{thm:state_estimation_substitution_error}, we consider
different ratios between the transition-noise covariance \(\mathbf{W}_T\) and
the observation one \(\mathbf{V}_T\). The base covariance matrix is
set to \(
\mathbf{W}_{\mathrm{base}}
=
\operatorname{diag}\!\left(
0.0196,\,
0.0196,\,
0.0100,\,
0.0100
\right).
\) 

Experiments also include an \(\epsilon\)-\textit{perturbation} baseline to provide a fair comparison. In this setting, perturbations are applied with the same probability as the adversarial attacks
but are sampled randomly from the same \(\epsilon\)-ellipsoidal feasible
region. This allows us to distinguish the effect of adversarial optimization
from that of merely adding perturbations of the same magnitude.
Moreover, we consider a \emph{noiseless} baseline, in which the agent acts on the
true state, and a \emph{noisy + KF} baseline, in which the agent acts on the
KF estimate with no adversarial perturbations.

Results are shown in \autoref{fig:covariance_sweep_RL}, where effectiveness is measured in terms of mean cumulative reward. In the leftmost setting, where the observation noise is considerably larger than the transition noise, the attack based on \(m_T\) yields a higher cumulative reward than the oracle attack based on \(s_T^0\). This indicates that the posterior mean does not provide a sufficiently accurate approximation of the actual state. Nevertheless, the attack based on \(m_T\) remains more effective than the random \(\epsilon\)-perturbation baseline. On the other hand, when the observation and transition covariances have the same magnitude, or when the observation noise is smaller than the transition noise, the attack based on \(m_T\) performs similarly to the oracle attack, as can be seen in the middle and right panels of \autoref{fig:covariance_sweep_RL}. These results indicate that the approximate objective closely reproduces the ideal objective when the observations are sufficiently informative relative to the uncertainty in the transition dynamics.

\begin{figure}[t]
    \centering
    \includegraphics[width=0.95\linewidth]{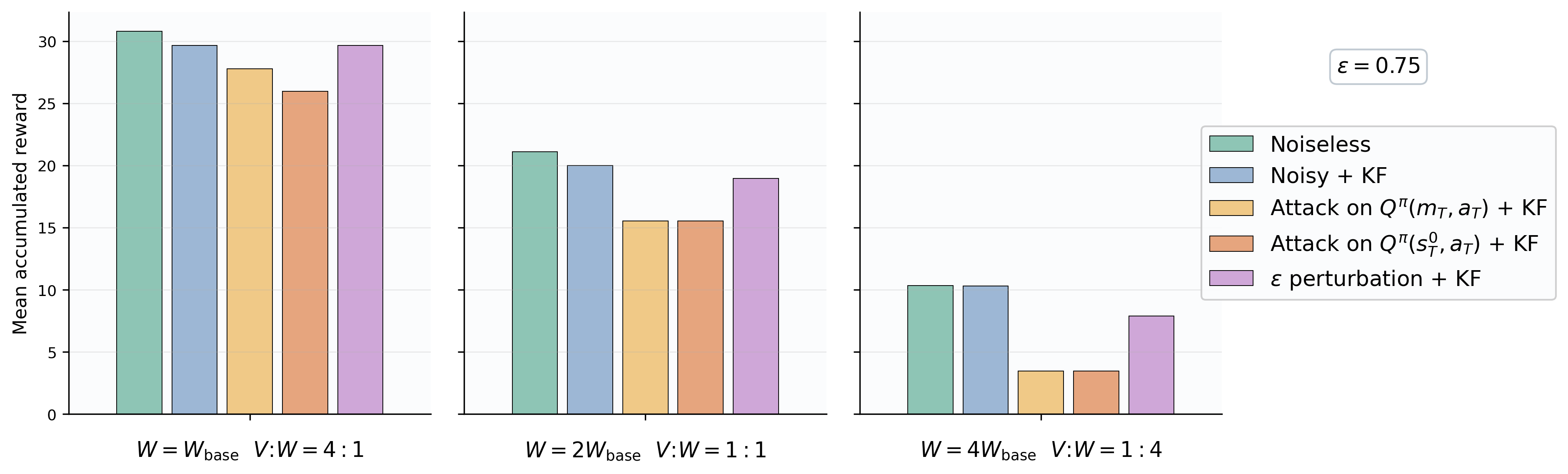}
    \caption{Mean cumulative reward over \(100\) episodes under different
    uncertainty configurations for constraints of $\epsilon = 0.75$. \textbf{(Left)}
    \(\mathbf{W}_T=\mathbf{W}_{\mathrm{base}}\) and
    \(\mathbf{V}_T=4\mathbf{W}_{\mathrm{base}}\).
    \textbf{(Middle)}
    \(\mathbf{W}_T=2\mathbf{W}_{\mathrm{base}}\) and
    \(\mathbf{V}_T=2\mathbf{W}_{\mathrm{base}}\).
    \textbf{(Right)}
    \(\mathbf{W}_T=4\mathbf{W}_{\mathrm{base}}\) and
    \(\mathbf{V}_T=\mathbf{W}_{\mathrm{base}}\).}
    \label{fig:covariance_sweep_RL}
\end{figure}

For further insight of the effect of the covariance matrices involved and for different epsilons, check \autoref{app:experimental}.

Once the effect of the covariance matrices on attack effectiveness has been examined, we set the observation-transition covariance ratio to \(1{:}1\), with \(\mathbf{W}_T=\mathbf{V}_T=2\mathbf{W}_{\mathrm{base}}\). Moreover, from this point onward, we consider only the practical attack based on the posterior mean \(m_T\).

We next evaluate the performance of the defenses under this uncertainty setting. The \emph{attack + KF} setting defined by \eqref{eq:practical_rl_attack_objective} with probability \(p_{\mathrm{atk}}=0.1\) is performed, again considering the \emph{\(\epsilon\)-perturbation + KF} setting for a fair comparison. Additionally, \emph{noiseless} and \emph{noisy + KF} settings are included for benchmarking purposes.

We additionally evaluate the directional covariance-adaptation defense proposed in \autoref{sec:OnlineBayesianDefense}. The method is tested using \(\lambda=c\lambda_{\max}(\mathbf{R})\), with \(c\in\{0.5,1,2\}\), where \(\lambda_{\max}(\mathbf{R})\) denotes the largest eigenvalue of the observation covariance matrix. We also compare the proposed defense with the outlier-robust KF in \cite{duranmartin2024outlier}, considering both the inverse-multiquadric variant, denoted by WoLF-IMQ, and the thresholded Mahalanobis-distance variant, denoted by WoLF-TMD.

\begin{figure}[t]
    \centering
    \includegraphics[width=0.98\linewidth]{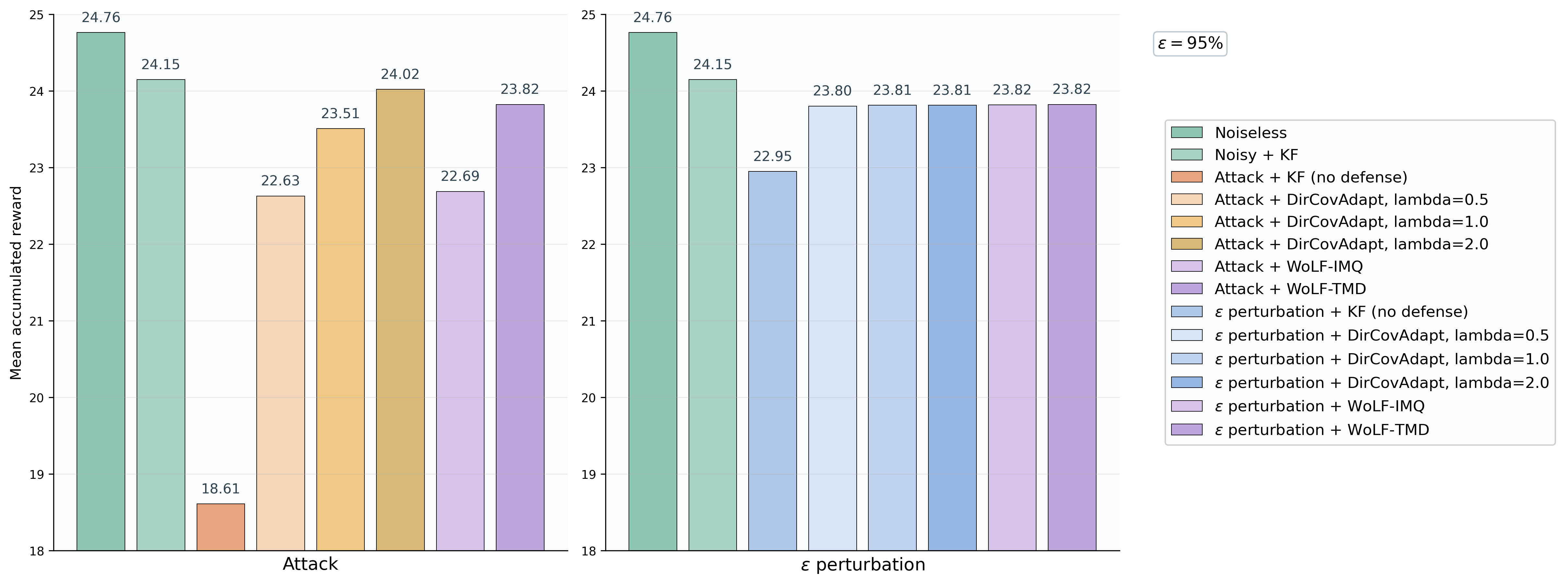}
    \caption{Mean cumulative reward over \(250\) episodes. \textbf{Left:} Performance of \emph{noiseless}, \emph{noisy + KF}, and \emph{attack + KF} settings, together with the proposed \textit{directional covariance-adaptation} defense and the \textit{WoLF} baselines under adversarial attacks. \textbf{Right:} Performance of the \emph{noiseless}, \emph{noisy + KF}, and \emph{\(\epsilon\)-perturbation + KF} settings, together with the proposed \textit{directional covariance-adaptation} defense and the \textit{WoLF} baselines under random \(\epsilon\)-perturbations.}
    \label{fig:mean_cumulative_rewards_95}
\end{figure}

The mean cumulative rewards obtained over \(250\) episodes are shown in
\autoref{fig:mean_cumulative_rewards_95} for \(\epsilon=0.95\). The corresponding
results for \(\epsilon=0.75\) are provided in
\autoref{app:experimental}. As expected, the \emph{noiseless}
setting achieves the highest reward because the policy receives the uncorrupted
state. It is followed by the \emph{noisy + KF} setting, in which the KF partially compensates for the transition and observation noise.

The \emph{attack + KF} setting consistently yields lower rewards than both the
\emph{noisy + KF} and \emph{\(\epsilon\)-perturbation + KF} settings. This
performance gap shows that the direction of the perturbation, rather than only
its magnitude, is crucial for disrupting the agent's behaviour. Further
evidence is provided in \autoref{fig:perturbation_norm}, which compares the
perturbation norms produced by the optimized attack and the random
\(\epsilon\)-perturbation baseline. The optimized attack does not necessarily
lie on the boundary of the feasible ellipsoid and can therefore have  smaller
euclidean norm than the random perturbations. Nevertheless, it produces a
larger reduction in cumulative reward. This confirms that a carefully selected
perturbation direction can be more damaging than random noise of greater
magnitude.

Regarding defense performance, the left panel of
\autoref{fig:mean_cumulative_rewards_95} shows the effect of the different
methods under adversarial observations. For the proposed covariance-adaptation
defense, increasing \(\lambda\) generally improves the cumulative reward, since
stronger covariance inflation reduces the influence of the potentially harmful
observation. In contrast, WoLF-IMQ and WoLF-TMD downweight observations
according to their discrepancy from the predicted measurement, without
accounting for how the perturbation direction affects the critic. Consequently,
adversarial observations that remain plausible under the nominal observation
model may not be sufficiently discounted. Moreover, by modifying only the
covariance along the estimated harmful direction, the proposed defense preserves
the information accumulated along orthogonal directions.

This difference in performance dissapears in the right panel of \autoref{fig:mean_cumulative_rewards_95}, where the defenses
are evaluated under random \(\epsilon\)-perturbations. Here, both directional
covariance adaptation, WoLF-TMD and WoLF-IMQ perform well as they can strongly
discount perturbations that appear anomalous under the nominal model.

Finally, the behavior of the covariance-adaptation defense differs depending on whether the observations are generated by the optimized attack or by random \(\epsilon\)-perturbations. This is illustrated in the left panel of \autoref{fig:perturbation_norm}, which shows the distribution of the posterior attack probabilities \(\gamma_t\). Under the optimized attack, these probabilities are more broadly distributed, indicating that the defense identifies different levels of adversarial risk according to the evidence provided by the weighted experts. In contrast, under random \(\epsilon\)-perturbations sampled near the boundary of the feasible region, \(\gamma_t\) is generally concentrated closer to zero. This difference suggests that the distribution of \(\gamma_t\) could help distinguish between deliberately optimized attacks and unusual large random observation noise.

\begin{figure}[t]
    \centering
    \includegraphics[width=0.8\linewidth]{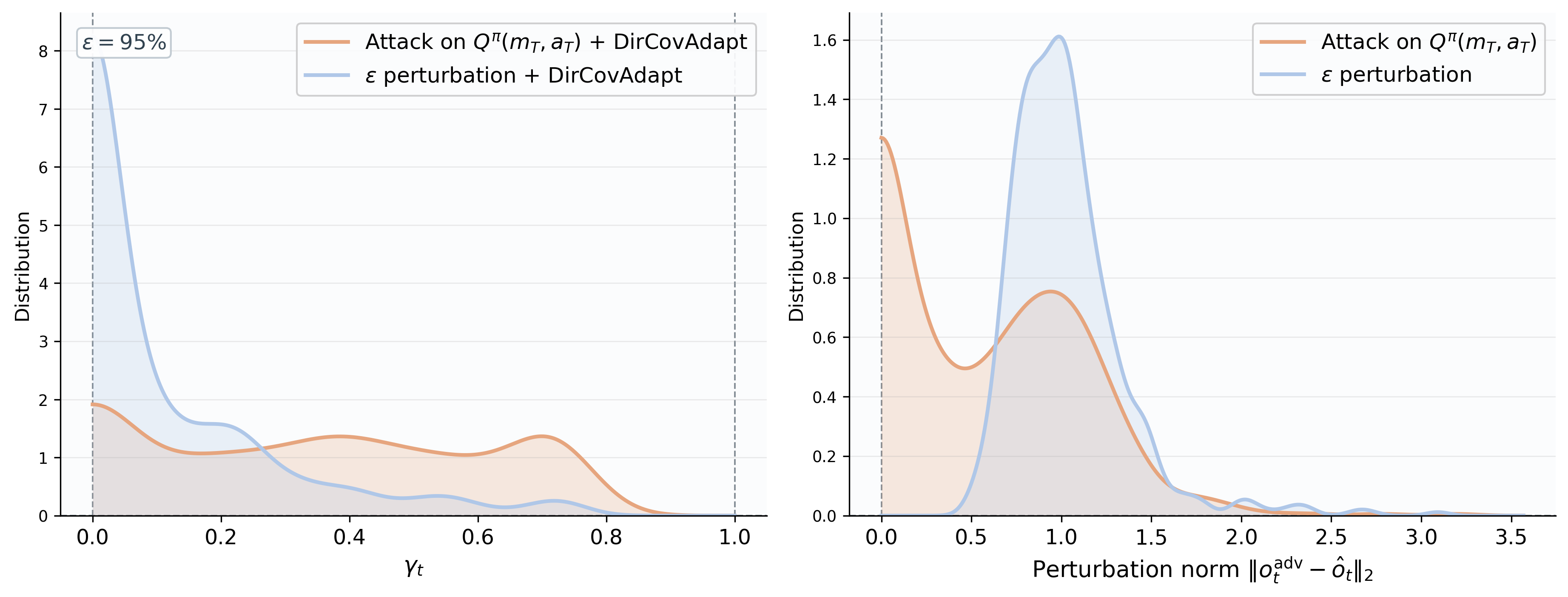}
    \caption{Comparison of the \textit{attack + KF} and
    \textit{\(\epsilon\)-perturbation + KF} perturbations.
    \textbf{Left:} Defense activation score \(\gamma_t\) under directional
    covariance adaptation.
    \textbf{Right:} Distribution of the perturbation magnitude
    \(\lVert o_t^{\mathrm{adv}}-\hat{o}_t\rVert_2\).}
    \label{fig:perturbation_norm}
\end{figure}

\FloatBarrier

\subsection{\textbf{Adversarial attacks in non-linear RL settings}}
\label{subsec:attack_rl_settings_nonlinear}

To further evaluate the proposed approach, we consider CartPole, a standard nonlinear RL benchmark with a four-dimensional state space. The experiments use the \texttt{CartPole-v1} environment from \texttt{Gymnasium} to generate the ground-truth dynamics, while the control policy is a pretrained Stable-Baselines3 DQN agent loaded from the public \texttt{sb3/dqn-CartPole-v1} checkpoint. The environment implementation and pretrained policy are available at
\hyperlink{https://gymnasium.farama.org/environments/classic_control/cart_pole}{link}
and \hyperlink{https://huggingface.co/sb3/dqn-CartPole-v1}{link},
respectively.

Unlike the previous experiments, the CartPole dynamics do not define a DLM. Instead, we construct a local state-space approximation from the nonlinear dynamics implemented by Gymnasium, using the same Euler discretization step as the simulator. The state is defined as
\[
s_t
=
\begin{bmatrix}
x_t &
\dot{x}_t &
\theta_t &
\dot{\theta}_t
\end{bmatrix}^{\top},
\]
where \(x_t\) and \(\dot{x}_t\) denote the cart position and velocity, while \(\theta_t\) and \(\dot{\theta}_t\) denote the pole angle and angular velocity.

Let \(f(s_t,u_t)\) denote the nonlinear discrete-time transition function. Following the standard Extended KF formulation \citep{murphy2023}, the dynamics are locally linearized around the current posterior mean and control input. Thus, the EKF replaces the nonlinear transition with a time-varying linear approximation that is recomputed at each time step. Since the complete CartPole state is observed with additive noise, the observation model is linear, with \(\mathbf{F}_t=\mathbf{I}\).

\begin{figure}[t]
    \centering
    \includegraphics[width=0.95\linewidth]{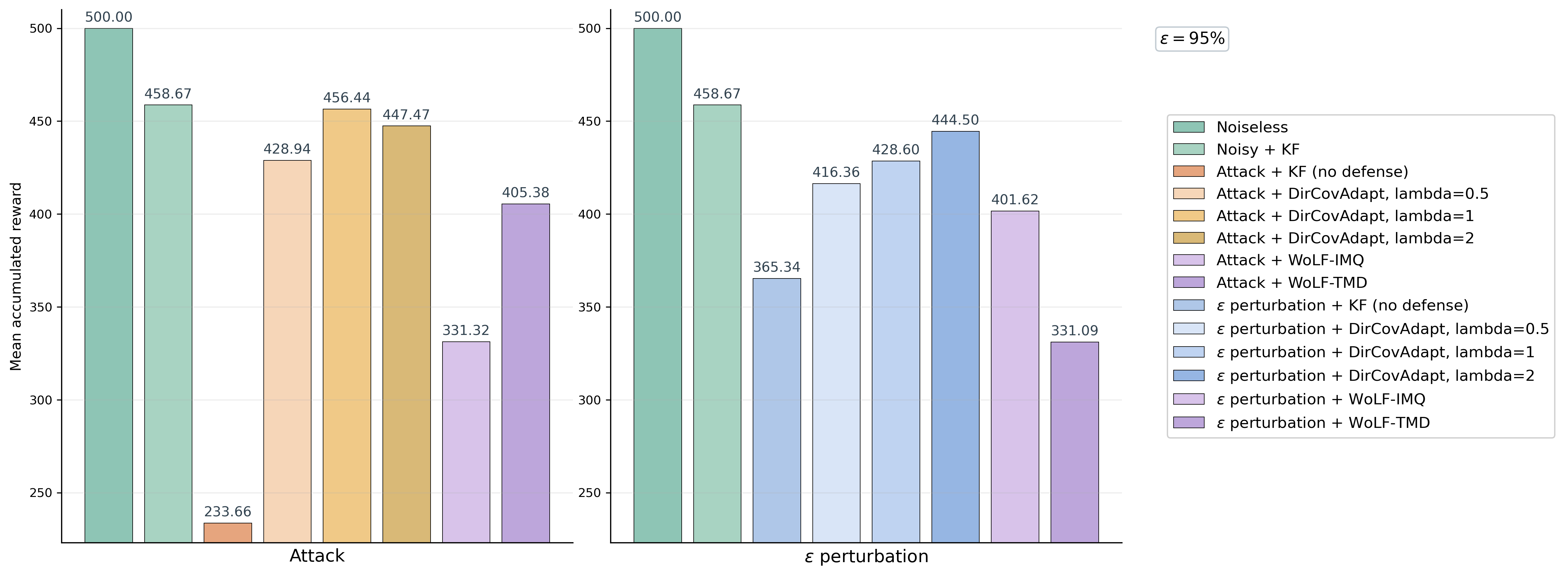}
    \caption{Mean cumulative reward over \(250\) episodes for the two perturbation settings in the CartPole environment. \textbf{Left:} Performance of the \emph{noiseless}, \emph{noisy + EKF}, and \emph{attack + EKF} settings, together with the proposed directional covariance-adaptation defense and the WoLF baselines under adversarial attacks. \textbf{Right:} Performance of the \emph{noiseless}, \emph{noisy + EKF}, and \emph{\(\epsilon\)-perturbation + EKF} settings, together with the proposed directional covariance-adaptation defense and the WoLF baselines under random \(\epsilon\)-perturbations.}
    \label{fig:cartpole_benchmark}
\end{figure}

Moreover, the process covariance \(\mathbf{W}_t\) represents not only aleatoric uncertainty, but also the epistemic one introduced by the local lineal approximation. It can therefore be interpreted as an effective representation of both process noise and uncertainty arising from the approximate transition model. This makes this setting particularly relevant for evaluating the proposed attack and defense, as the transition-noise variance is ten times higher than the observation one.

The results are shown in \autoref{fig:cartpole_benchmark}. The practical attack based on \(m_T\) is particularly effective in this setting. Since the transition model is approximate whereas the observation model remains directly informative, most of the additional uncertainty is assigned to the transition dynamics. Consequently, the posterior estimate is strongly informed by the observations, allowing the approximate attack objective to closely represent the objective based on the true state. As shown in the left panel, the optimized attack produces a substantially lower cumulative reward than both the \emph{noiseless} baseline and the random \(\epsilon\)-perturbation setting.

Regarding defense performance, directional covariance adaptation achieves its best result for
\(\lambda=2\lambda_{\max}(\mathbf{R})\). This stronger inflation reduces the influence of the adversarial observation only along the estimated harmful direction. In contrast to full covariance inflation, information accumulated along orthogonal directions is preserved, resulting in a more stable state estimate despite the large transition and model uncertainty.

Overall, these results indicate that directional covariance adaptation is especially beneficial when transition uncertainty is high. By avoiding unnecessary covariance inflation along unaffected directions, the defense preserves previously accumulated information and maintains a more stable state estimate.

\FloatBarrier

\section{Discussion} 

This work has developed a strategy for identifying adversarial attacks in the observation space under probabilistic plausibility constraints. This allows us to determine the directions that have the greatest effect on the inferred state while remaining plausible under the predictive model. The proposed framework extends standard adversarial machine learning formulations by replacing direct perturbation bounds with a likelihood constraint, making the attacks less likely to be detected by model-based anomaly-detection systems. Moreover, by exploiting the sequential structure of SSMs, the attack accounts for the information propagated through both the transition and observation
models, rather than considering the attacked observation model in isolation. The methodology can also be adapted to different attack objectives and was successfully extended to RL.

In addition to the attack formulation, we developed a defense mechanism for
online inference. The defense is
formulated within an online Bayesian adaptation framework that combines prior
beliefs about the occurrence of an attack with the evidence supplied by the
current observation. When an attack is considered likely, the observation
covariance is inflated selectively along the estimated adversarial direction. This directional adaptation allows the filter to discount harmful information to the decision objective, rather than modifying the entire covariance matrix. In our experiments, this task-dependent defense performed comparably to the considered state-of-the-art robust filtering methods in some settings, while outperforming them when the observation noise was substantially smaller than the transition noise.

The main limitation of the proposed method is its additional computational
cost. For a dense system of dimension $d$, WoLF \citep{duranmartin2024outlier} has complexity $O(d^3)$,
which is the same order as the standard KF and mainly arises from
the inversion of dense covariance matrices.
By contrast, our method additionally solves an inner projected-gradient
optimization problem. With $N_{\mathrm{PGD}}$ iterations, its approximate
cost is
\[
O\!\left(d^3+N_{\mathrm{PGD}}d^2\right),
\]
where the $O(d^3)$ term corresponds to the required matrix inversion or
factorization, while each gradient and projection step costs approximately
$O(d^2)$. Thus, the increased decision-specific robustness is obtained at a
higher computational cost, which may become relevant in high-dimensional or
strict real-time applications.

Several directions remain for future work. First, the framework could be extended to attacks applied at multiple time steps. As in the RL setting, attacks that are sufficiently separated in time may be treated approximately independently, whereas nearby attacks should be optimized jointly to account for their temporal dependence. 

Second, this work considers only white-box attacks. A natural next step is to account for uncertainty in the SSM parameters, known as gray-box settings. By placing a distribution over model parameterizations, both the plausibility constraint and the attack objective could be averaged across models, thereby accounting for parameter uncertainty within a single optimization problem.

\section*{Funding}

M. Santos-Pascual acknowledges the Spanish Ministry of Science,
Innovation and Universities for the FPU24/04169 Ph.D. scholarship.



\bibliographystyle{apacite}
\bibliography{references}


\newpage

\appendix

\section{Proofs}
\label{app:affine_representations}

This appendix provides the explicit definitions of the coefficient matrices
introduced in Lemmas~\ref{lem:past_unroll} and
\ref{lem:explicit_o_a_rewrite}, together with the auxiliary derivations.

The system matrices and the sequences of KF gains
$\{K_u\}_{u=1}^{T}$ and RTS smoothing gains
$\{J_u\}_{u=1}^{T-1}$ are fixed with respect to the observation values under consideration. Moreover, we use the conventions
\begin{equation}
\overleftarrow{\prod_{j=p}^{q}}\mathbf M_j
:=
\mathbf M_q\mathbf M_{q-1}\cdots\mathbf M_p,
\qquad
\overrightarrow{\prod_{j=p}^{q}}J_j
:=
J_pJ_{p+1}\cdots J_q,
\label{eq:app_ordered_products}
\end{equation}
for $p\leq q$. In both cases, an empty product is defined as the identity
matrix:
\begin{equation}
\overleftarrow{\prod_{j=p}^{q}}\mathbf M_j
=
\overrightarrow{\prod_{j=p}^{q}}J_j
:=
\mathbf I,
\qquad
p>q.
\label{eq:app_empty_products}
\end{equation}

\subsection{Proof of \autoref{subsec:lgssm_attack}}
\label{app:proofslgssm_attack}

\begin{proof}[Proof of Lemma~\ref{lem:past_unroll}]
For $t\in\{1,\ldots,T\}$ and $k\in\{0,\ldots,T-t\}$, we show that
\begin{equation}
m_{t+k}
=
\left(\overleftarrow{\prod_{j=t}^{t+k}}\mathbf M_j\right)m_{t-1}
+
\sum_{i=0}^{k}
\left(\overleftarrow{\prod_{j=t+i+1}^{t+k}}\mathbf M_j\right)K_{t+i}o_{t+i}
+
\sum_{i=0}^{k}
\left(\overleftarrow{\prod_{j=t+i+1}^{t+k}}\mathbf M_j\right)\mathbf U_{t+i}a_{t+i-1}.
\label{eq:app_filtered_explicit}
\end{equation}
We proceed by induction on $k$. For $k=0$, using
$\overleftarrow{\prod_{j=t}^{t}}\mathbf M_j=\mathbf M_t$ and
$\overleftarrow{\prod_{j=t+1}^{t}}\mathbf M_j=\mathbf I$, the right-hand
side of \eqref{eq:app_filtered_explicit} reduces to
\[
m_t=\mathbf M_tm_{t-1}+\mathbf U_ta_{t-1}+K_to_t,
\]
which is exactly the one-step affine filtering recursion.

Assume now that \eqref{eq:app_filtered_explicit} holds for $k-1$. Applying
the filtering recursion at time $t+k$ gives
\[
m_{t+k}
=
\mathbf M_{t+k}m_{t+k-1}
+
\mathbf U_{t+k}a_{t+k-1}
+
K_{t+k}o_{t+k}.
\]
Substituting the induction hypothesis and using
\[
\mathbf M_{t+k}
\left(\overleftarrow{\prod_{j=p}^{t+k-1}}\mathbf M_j\right)
=
\overleftarrow{\prod_{j=p}^{t+k}}\mathbf M_j,
\]
we obtain
\begin{align}
m_{t+k}
&=
\left(
\overleftarrow{\prod_{j=t}^{t+k}}\mathbf M_j
\right)m_{t-1}
\nonumber\,+
\sum_{i=0}^{k-1}
\left(
\overleftarrow{\prod_{j=t+i+1}^{t+k}}\mathbf M_j
\right)
K_{t+i}o_{t+i}
\nonumber\\
&\quad+
\sum_{i=0}^{k-1}
\left(
\overleftarrow{\prod_{j=t+i+1}^{t+k}}\mathbf M_j
\right)
\mathbf U_{t+i}a_{t+i-1}
\nonumber\,+
K_{t+k}o_{t+k}
+
\mathbf U_{t+k}a_{t+k-1}.
\label{eq:app_filtered_induction_step}
\end{align}
Since
$\overleftarrow{\prod_{j=t+k+1}^{t+k}}\mathbf M_j=\mathbf I$, the last
observation and action terms are precisely the terms corresponding to $i=k$.
Therefore, \eqref{eq:app_filtered_explicit} also holds for $k$.

The coefficient matrices introduced in Lemma~\ref{lem:past_unroll} are
\begin{equation}
\boldsymbol{\Phi}_{t,k}
:=
\overleftarrow{\prod_{j=t}^{t+k}}\mathbf M_j,
\label{eq:app_Phi_filtered_definition}
\end{equation}
\begin{equation}
\boldsymbol{\Gamma}^{(o)}_{t,k,i}
:=
\left(\overleftarrow{\prod_{j=t+i+1}^{t+k}}\mathbf M_j\right)K_{t+i},
\qquad i=0,\ldots,k,
\label{eq:app_Gamma_o_definition}
\end{equation}
and
\begin{equation}
\boldsymbol{\Gamma}^{(a)}_{t,k,i}
:=
\left(\overleftarrow{\prod_{j=t+i+1}^{t+k}}\mathbf M_j\right)\mathbf U_{t+i},
\qquad i=0,\ldots,k.
\label{eq:app_Gamma_a_definition}
\end{equation}
Substituting these definitions into \eqref{eq:app_filtered_explicit} gives
\[
m_{t+k}
=
\boldsymbol{\Phi}_{t,k}m_{t-1}
+
\sum_{i=0}^{k}\boldsymbol{\Gamma}^{(o)}_{t,k,i}o_{t+i}
+
\sum_{i=0}^{k}\boldsymbol{\Gamma}^{(a)}_{t,k,i}a_{t+i-1},
\]
which proves Lemma~\ref{lem:past_unroll}.
\end{proof}


Recall that the RTS recursion can be written as
\begin{equation}
m_{u|T}
=
\mathbf Q_um_u
+
J_um_{u+1|T}
-
J_u\mathbf B_{u+1}a_u,
\qquad
\mathbf Q_u:=\mathbf I-J_u\mathbf A_{u+1},
\label{eq:app_rts_affine_recursion}
\end{equation}
with terminal condition $m_{T|T}=m_T$ and $\mathbf Q_T:=\mathbf I$.

Using the convention
\[
\overrightarrow{\prod_{j=t}^{r-1}}J_j
=
J_tJ_{t+1}\cdots J_{r-1},
\qquad
\overrightarrow{\prod_{j=t}^{t-1}}J_j
=
\mathbf I,
\]
we unroll \eqref{eq:app_rts_affine_recursion} backwards from time $t$. The first substitution gives
\[
m_{t|T}
=
\mathbf Q_tm_t
+
J_tm_{t+1|T}
-
J_t\mathbf B_{t+1}a_t.
\]
Applying the same recursion to $m_{t+1|T}$ gives
\[
m_{t|T}
=
\mathbf Q_tm_t
+
J_t\mathbf Q_{t+1}m_{t+1}
+
J_tJ_{t+1}m_{t+2|T}
-
J_t\mathbf B_{t+1}a_t
-
J_tJ_{t+1}\mathbf B_{t+2}a_{t+1}.
\]
More generally, after unrolling the recursion up to
$q\in\{t,\ldots,T-1\}$, we obtain
\begin{equation}
m_{t|T}
=
\sum_{r=t}^{q}
\left(\overrightarrow{\prod_{j=t}^{r-1}}J_j\right)\mathbf Q_rm_r
+
\left(\overrightarrow{\prod_{j=t}^{q}}J_j\right)m_{q+1|T}
-
\sum_{r=t}^{q}
\left(\overrightarrow{\prod_{j=t}^{r-1}}J_j\right)J_r\mathbf B_{r+1}a_r.
\label{eq:app_rts_partial_unrolling}
\end{equation}
Indeed, substituting
\[
m_{q+1|T}
=
\mathbf Q_{q+1}m_{q+1}
+
J_{q+1}m_{q+2|T}
-
J_{q+1}\mathbf B_{q+2}a_{q+1}
\]
into \eqref{eq:app_rts_partial_unrolling} and using
\[
\left(\overrightarrow{\prod_{j=t}^{q}}J_j\right)J_{q+1}
=
\overrightarrow{\prod_{j=t}^{q+1}}J_j
\]
extends both sums from $q$ to $q+1$. Setting $q=T-1$ in \eqref{eq:app_rts_partial_unrolling} gives
\begin{equation}
m_{t|T}
=
\sum_{r=t}^{T-1}
\left(\overrightarrow{\prod_{j=t}^{r-1}}J_j\right)\mathbf Q_rm_r
+
\left(\overrightarrow{\prod_{j=t}^{T-1}}J_j\right)m_{T|T}
-
\sum_{r=t}^{T-1}
\left(\overrightarrow{\prod_{j=t}^{r-1}}J_j\right)J_r\mathbf B_{r+1}a_r.
\label{eq:app_rts_terminal_expansion}
\end{equation}
Since $m_{T|T}=m_T$ and $\mathbf Q_T=\mathbf I$, the terminal term can be
incorporated into the first sum. Therefore,
\begin{equation}
m_{t|T}
=
\sum_{r=t}^{T}
\left(\overrightarrow{\prod_{j=t}^{r-1}}J_j\right)\mathbf Q_rm_r
-
\sum_{r=t}^{T-1}
\left(\overrightarrow{\prod_{j=t}^{r-1}}J_j\right)J_r\mathbf B_{r+1}a_r.
\label{eq:app_rts_fully_unrolled}
\end{equation}

\begin{proof}[Proof of Lemma~\ref{lem:explicit_o_a_rewrite}]
For every $r\in\{t,\ldots,T\}$, the filtered expansion
\eqref{eq:app_filtered_explicit}, applied with $k=r-t$, gives
\begin{equation}
m_r
=
\left(\overleftarrow{\prod_{j=t}^{r}}\mathbf M_j\right)m_{t-1}
+
\sum_{u=t}^{r}
\left(\overleftarrow{\prod_{j=u+1}^{r}}\mathbf M_j\right)K_uo_u
+
\sum_{u=t}^{r}
\left(\overleftarrow{\prod_{j=u+1}^{r}}\mathbf M_j\right)\mathbf U_ua_{u-1}.
\label{eq:app_filtered_mean_for_rts}
\end{equation}
Substituting \eqref{eq:app_filtered_mean_for_rts} into
\eqref{eq:app_rts_fully_unrolled}, the coefficient multiplying $m_{t-1}$ is
\begin{equation}
\mathbf H^{(m)}_t
:=
\sum_{r=t}^{T}
\left(\overrightarrow{\prod_{j=t}^{r-1}}J_j\right)
\mathbf Q_r
\left(\overleftarrow{\prod_{j=t}^{r}}\mathbf M_j\right).
\label{eq:app_Hm_definition}
\end{equation}
The observation terms initially have the form
\[
\sum_{r=t}^{T}
\sum_{u=t}^{r}
\left(\overrightarrow{\prod_{j=t}^{r-1}}J_j\right)
\mathbf Q_r
\left(\overleftarrow{\prod_{j=u+1}^{r}}\mathbf M_j\right)K_uo_u.
\]
Exchanging the order of summation over the region
$t\leq u\leq r\leq T$ gives
\[
\sum_{u=t}^{T}
\left[
\sum_{r=u}^{T}
\left(\overrightarrow{\prod_{j=t}^{r-1}}J_j\right)
\mathbf Q_r
\left(\overleftarrow{\prod_{j=u+1}^{r}}\mathbf M_j\right)K_u
\right]o_u.
\]
Hence,
\begin{equation}
\mathbf H^{(o)}_{t,u}
:=
\sum_{r=u}^{T}
\left(\overrightarrow{\prod_{j=t}^{r-1}}J_j\right)
\mathbf Q_r
\left(\overleftarrow{\prod_{j=u+1}^{r}}\mathbf M_j\right)K_u,
\qquad
u=t,\ldots,T.
\label{eq:app_Ho_definition}
\end{equation}
For the actions inherited from the filtering recursion, let $v=u-1$. Their
contribution becomes
\[
\sum_{v=t-1}^{T-1}
\left[
\sum_{r=v+1}^{T}
\left(\overrightarrow{\prod_{j=t}^{r-1}}J_j\right)
\mathbf Q_r
\left(\overleftarrow{\prod_{j=v+2}^{r}}\mathbf M_j\right)\mathbf U_{v+1}
\right]a_v.
\]
The direct RTS contribution is
\[
-
\sum_{v=t}^{T-1}
\left(\overrightarrow{\prod_{j=t}^{v-1}}J_j\right)
J_v\mathbf B_{v+1}a_v.
\]
Combining both contributions, for $v=t-1,\ldots,T-1$ define
\begin{equation}
\mathbf H^{(a)}_{t,v}
:=
\sum_{r=v+1}^{T}
\left(\overrightarrow{\prod_{j=t}^{r-1}}J_j\right)
\mathbf Q_r
\left(\overleftarrow{\prod_{j=v+2}^{r}}\mathbf M_j\right)\mathbf U_{v+1}
-
\mathbf 1_{\{v\geq t\}}
\left(\overrightarrow{\prod_{j=t}^{v-1}}J_j\right)
J_v\mathbf B_{v+1}.
\label{eq:app_Ha_definition}
\end{equation}
The indicator removes the direct RTS correction when $v=t-1$, since that
correction begins with the action $a_t$.

Collecting the contributions of $m_{t-1}$, the observations, and the actions
gives
\begin{equation}
m_{t|T}
=
\mathbf H^{(m)}_tm_{t-1}
+
\sum_{u=t}^{T}\mathbf H^{(o)}_{t,u}o_u
+
\sum_{v=t-1}^{T-1}\mathbf H^{(a)}_{t,v}a_v.
\label{eq:app_rts_final_representation}
\end{equation}
This is the representation stated in
Lemma~\ref{lem:explicit_o_a_rewrite}.
\end{proof}

\begin{proof}[Proof of Theorem~\ref{thm:naive_qcqp}]
By Lemma~\ref{lem:explicit_o_a_rewrite}, the RTS smoothed mean can be written
as
\begin{equation}
m_{t|T}
=
\mathbf H^{(m)}_tm_{t-1}
+
\sum_{u=t}^{T}\mathbf H^{(o)}_{t,u}o_u
+
\sum_{u=t-1}^{T-1}\mathbf H^{(a)}_{t,u}a_u.
\label{eq:app_affine_rts_recalled}
\end{equation}
Replacing only $o_t$ by $o_t^{\mathrm{adv}}$, while keeping all other
observations, actions, and $m_{t-1}$ fixed, gives
\begin{equation}
m_{t|T}^{\mathrm{adv}}
=
\mathbf H^{(m)}_tm_{t-1}
+
\mathbf H^{(o)}_{t,t}o_t^{\mathrm{adv}}
+
\sum_{u=t+1}^{T}\mathbf H^{(o)}_{t,u}o_u
+
\sum_{u=t-1}^{T-1}\mathbf H^{(a)}_{t,u}a_u.
\label{eq:app_adversarial_rts_mean}
\end{equation}
Subtracting \eqref{eq:app_affine_rts_recalled} from
\eqref{eq:app_adversarial_rts_mean}, all terms except the contribution of
$o_t$ cancel. Therefore,
\begin{equation}
m_{t|T}^{\mathrm{adv}}-m_{t|T}
=
\mathbf H^{(o)}_{t,t}
\left(o_t^{\mathrm{adv}}-o_t\right).
\label{eq:app_adversarial_mean_difference}
\end{equation}
The squared Euclidean displacement of the smoothed state estimate is thus
\begin{equation}
\left\|
m_{t|T}^{\mathrm{adv}}-m_{t|T}
\right\|_2^2
=
\left(o_t^{\mathrm{adv}}-o_t\right)^\top
\left(\mathbf H^{(o)}_{t,t}\right)^\top
\mathbf H^{(o)}_{t,t}
\left(o_t^{\mathrm{adv}}-o_t\right).
\label{eq:app_adversarial_mean_quadratic}
\end{equation}
Using the definition
\[
\mathbf R_t
=
\left(\mathbf H^{(o)}_{t,t}\right)^\top
\mathbf H^{(o)}_{t,t},
\]
we obtain
\begin{equation}
\left\|
m_{t|T}^{\mathrm{adv}}-m_{t|T}
\right\|_2^2
=
\left(o_t^{\mathrm{adv}}-o_t\right)^\top
\mathbf R_t
\left(o_t^{\mathrm{adv}}-o_t\right).
\label{eq:app_attack_objective}
\end{equation}
Moreover, $\mathbf R_t$ is positive semidefinite, since for every
$x\in\mathbb R^{d_o}$,
\[
x^\top\mathbf R_tx
=
x^\top
\left(\mathbf H^{(o)}_{t,t}\right)^\top
\mathbf H^{(o)}_{t,t}x
=
\left\|
\mathbf H^{(o)}_{t,t}x
\right\|_2^2
\geq0.
\]
By Lemma~\ref{lem:loo_pred_obs_past_future}, the leave-one-out predictive
density is
\begin{equation}
p(o_t\mid o_{-t},a_{0:T-1})
=
(2\pi)^{-d_o/2}|S_{-t}|^{-1/2}
\exp\left\{
-\frac{1}{2}
(o_t-\hat o_{-t})^\top
S_{-t}^{-1}
(o_t-\hat o_{-t})
\right\}.
\label{eq:app_loo_gaussian_density}
\end{equation}
Evaluating this density at $o_t^{\mathrm{adv}}$, the plausibility constraint
\begin{equation}
p(o_t^{\mathrm{adv}}\mid o_{-t},a_{0:T-1})
\geq\epsilon
\label{eq:app_likelihood_constraint}
\end{equation}
is equivalent to
\begin{equation}
(2\pi)^{-d_o/2}|S_{-t}|^{-1/2}
\exp\left\{
-\frac{1}{2}
(o_t^{\mathrm{adv}}-\hat o_{-t})^\top
S_{-t}^{-1}
(o_t^{\mathrm{adv}}-\hat o_{-t})
\right\}
\geq\epsilon.
\label{eq:app_likelihood_constraint_expanded}
\end{equation}
Taking logarithms gives
\begin{equation}
-\frac{d_o}{2}\log(2\pi)
-\frac{1}{2}\log|S_{-t}|
-\frac{1}{2}
(o_t^{\mathrm{adv}}-\hat o_{-t})^\top
S_{-t}^{-1}
(o_t^{\mathrm{adv}}-\hat o_{-t})
\geq
\log\epsilon.
\label{eq:app_log_likelihood_constraint}
\end{equation}
Multiplying by $-2$ and reversing the inequality yields
\begin{equation}
(o_t^{\mathrm{adv}}-\hat o_{-t})^\top
S_{-t}^{-1}
(o_t^{\mathrm{adv}}-\hat o_{-t})
\leq
-2\log\epsilon
-d_o\log(2\pi)
-\log|S_{-t}|.
\label{eq:app_ellipsoidal_constraint}
\end{equation}
By the definition of $\rho_\epsilon$ in
\eqref{eq:rho_epsilon_definition}, this becomes
\begin{equation}
(o_t^{\mathrm{adv}}-\hat o_{-t})^\top
S_{-t}^{-1}
(o_t^{\mathrm{adv}}-\hat o_{-t})
\leq
\rho_\epsilon.
\label{eq:app_ellipsoidal_constraint_rho}
\end{equation}
The condition
\[
\epsilon
\leq
(2\pi)^{-d_o/2}|S_{-t}|^{-1/2}
\]
ensures that $\rho_\epsilon\geq0$, and hence that the ellipsoidal feasible
region is nonempty.

Combining the quadratic objective
\eqref{eq:app_attack_objective} with the ellipsoidal constraint
\eqref{eq:app_ellipsoidal_constraint_rho} gives
\[
\max_{o_t^{\mathrm{adv}}\in\mathbb R^{d_o}}
\left(o_t^{\mathrm{adv}}-o_t\right)^\top
\mathbf R_t
\left(o_t^{\mathrm{adv}}-o_t\right)
\]
subject to
\[
(o_t^{\mathrm{adv}}-\hat o_{-t})^\top
S_{-t}^{-1}
(o_t^{\mathrm{adv}}-\hat o_{-t})
\leq
\rho_\epsilon,
\]
which is precisely the optimization problem in
\eqref{eq:naive_qcqp}.
\end{proof}

\subsection{Proof of \autoref{subsec:rl_attack}}
\label{app:state_estimation_substitution_error}

\begin{proof}[Proof of Theorem~\ref{thm:state_estimation_substitution_error}]
Under the clean posterior, define \(\delta_T
=
s_T^0-m_T.\)  Since \(m_T\) and \(P_T\) are the posterior mean and covariance,
\begin{equation}
\label{eq:posterior_error_moments}
\mathbb E
\left[
\delta_T
\right]
=
0,
\qquad
\mathbb E
\left[
\delta_T
\delta_T^\top
\right]
=
P_T.
\end{equation}
For any fixed action \(a\), Taylor's theorem with integral remainder gives \begin{equation}
Q^\pi
\left(
s_T^0,
a
\right)
-
Q^\pi
\left(
m_T,
a
\right)
=
\nabla_s
Q^\pi
\left(
m_T,
a
\right)^\top
\delta_T+
\int_0^1
(1-\tau)
\delta_T^\top
\nabla_{ss}^{2}
Q^\pi
\left(
m_T+\tau\delta_T,
a
\right)
\delta_T
\,d\tau.
\label{eq:exact_taylor_state_error}
\end{equation}
Taking expectation with respect to \(s_T^0\) eliminates the
first-order term as \(\mathbb E_{s_T^0}[\delta_T]=0\),
\begin{align}
\left|
\mathbb E_{s_T^0}
\left[
Q^\pi
\left(
s_T^0,
a
\right)
-
Q^\pi
\left(
m_T,
a
\right)
\right]
\right|
&\leq
\mathbb E_{s_T^0}
\left[
\int_0^1
(1-\tau)
\left|
\delta_T^\top
\nabla_{ss}^{2}
Q^\pi
\left(
m_T+\tau\delta_T,
a
\right)
\delta_T
\right|
d\tau
\right]
\nonumber\\
&\leq
M(a)
\int_0^1
(1-\tau)
\mathbb E_{s_T^0}
\left[
\left\|
\delta_T
\right\|^2
\right]
d\tau=
\frac{M(a)}{2}
\operatorname{tr}
\left(
P_T
\right)\nonumber\\
&\leq
\frac{M_{\max}}{2}
\operatorname{tr}
\left(
P_T
\right),
\label{eq:action_state_substitution_bound}
\end{align}
where we have used
\(
\left|
x^\top Bx
\right|
\leq
\left\|
B
\right\|_{\mathrm{op}}
\left\|
x
\right\|^2
\),
together with
\(
\left\|
\nabla_{ss}^{2}Q^\pi(s,a)
\right\|_{\mathrm{op}}
\leq
M(a)
\leq
M_{\max}
\),
and
\[
\mathbb E_{s_T^0}
\left[
\left\|
\delta_T
\right\|^2
\right]
=
\operatorname{tr}
\left(
\mathbb E_{s_T^0}
\left[
\delta_T\delta_T^\top
\right]
\right)
=
\operatorname{tr}
\left(
P_T
\right).
\]

We now take expectations with respect to the attacked state
\(s_T^{\mathrm{adv}}\) and the induced action \(a_T\). By the definition of
\(\overline{\mathcal E}_T(o_T^{\mathrm{adv}})\), Fubini's theorem and triangle inequality,
\begin{align}
\left|
\overline{\mathcal E}_T
\left(
o_T^{\mathrm{adv}}
\right)
\right|
\leq
\mathbb E_{
s_T^{\mathrm{adv}}
}
\Bigg[
\mathbb E_{
a_T
}
\Bigg[
\left|
\mathbb E_{
s_T^0
}
\Big[
Q^\pi
\left(
s_T^0,
a_T
\right)
-
Q^\pi
\left(
m_T,
a_T
\right)
\Big]
\right|
\Bigg]
\Bigg]\leq
\frac{M_{\max}}{2}
\operatorname{tr}
\left(
P_T
\right).
\label{eq:expected_state_substitution_bound}
\end{align}
Finally, the Kalman update gives
\begin{equation}
P_T
=
\left(
P_{T\mid T-1}^{-1}
+
\mathbf F_T^\top
\mathbf V_T^{-1}
\mathbf F_T
\right)^{-1}
=\left(
\left(
\mathbf A_T
P_{T-1}
\mathbf A_T^\top
+
\mathbf W_T
\right)^{-1}
+
\mathbf F_T^\top
\mathbf V_T^{-1}
\mathbf F_T
\right)^{-1}.
\label{eq:posterior_covariance_rl_attack}
\end{equation}

Substituting \eqref{eq:posterior_covariance_rl_attack} into the previous bound
proves the result.

\end{proof}

\FloatBarrier

\section{Additional experimental results}
\label{app:experimental}
This section provides additional figures that can help the reader gain further insight into the results of experiments in \autoref{subsec:attack_rl_settings}

\begin{figure}[ht]
    \centering
    \includegraphics[width=0.95\linewidth]{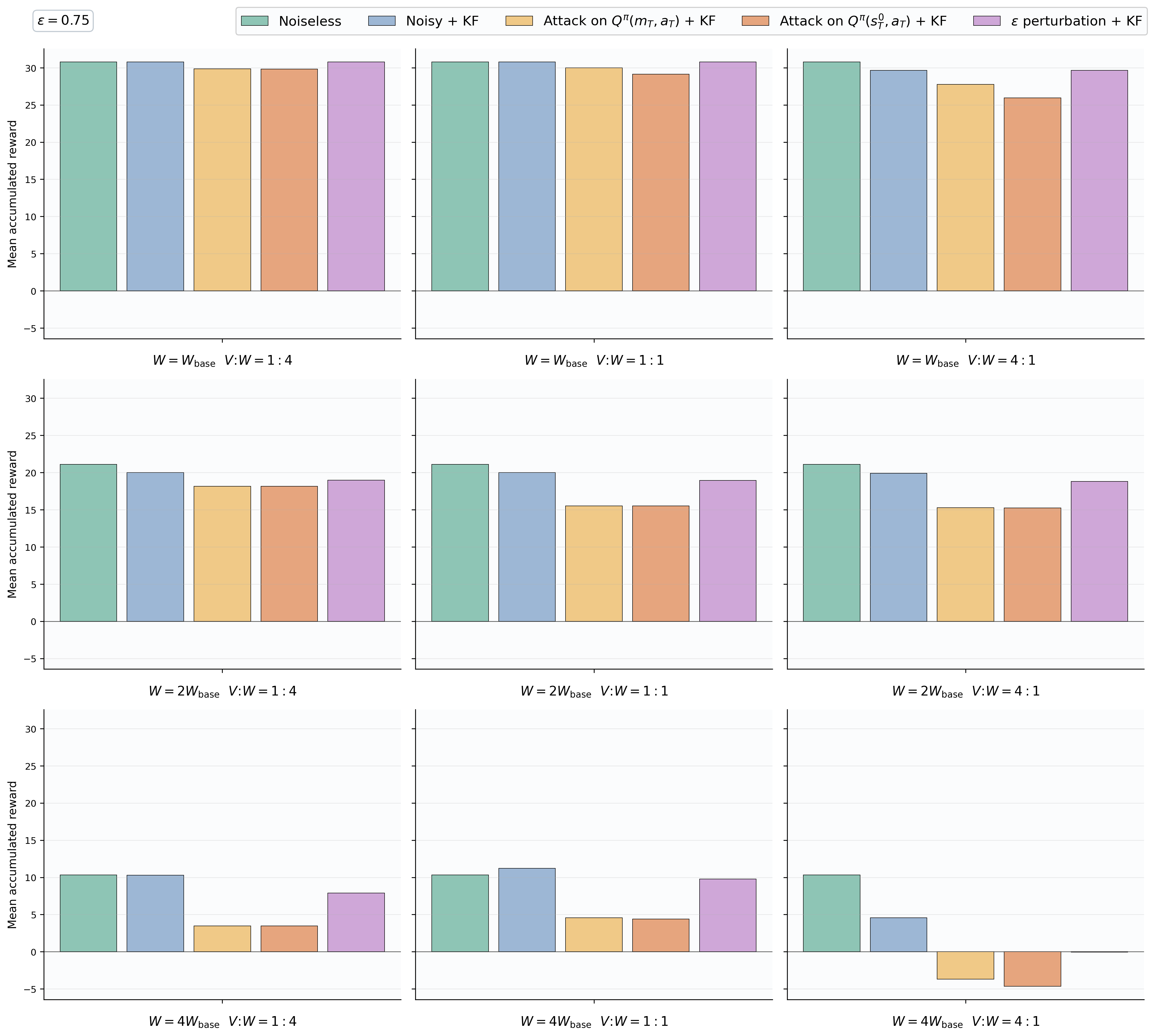}
    \caption{Mean cumulative reward over \(100\) episodes under different
    uncertainty configurations for constraints of $\epsilon = 0.75$.}
    \label{fig:95covaraicnesweep}
\end{figure}

\begin{figure}[ht]
    \centering
    \includegraphics[width=0.95\linewidth]{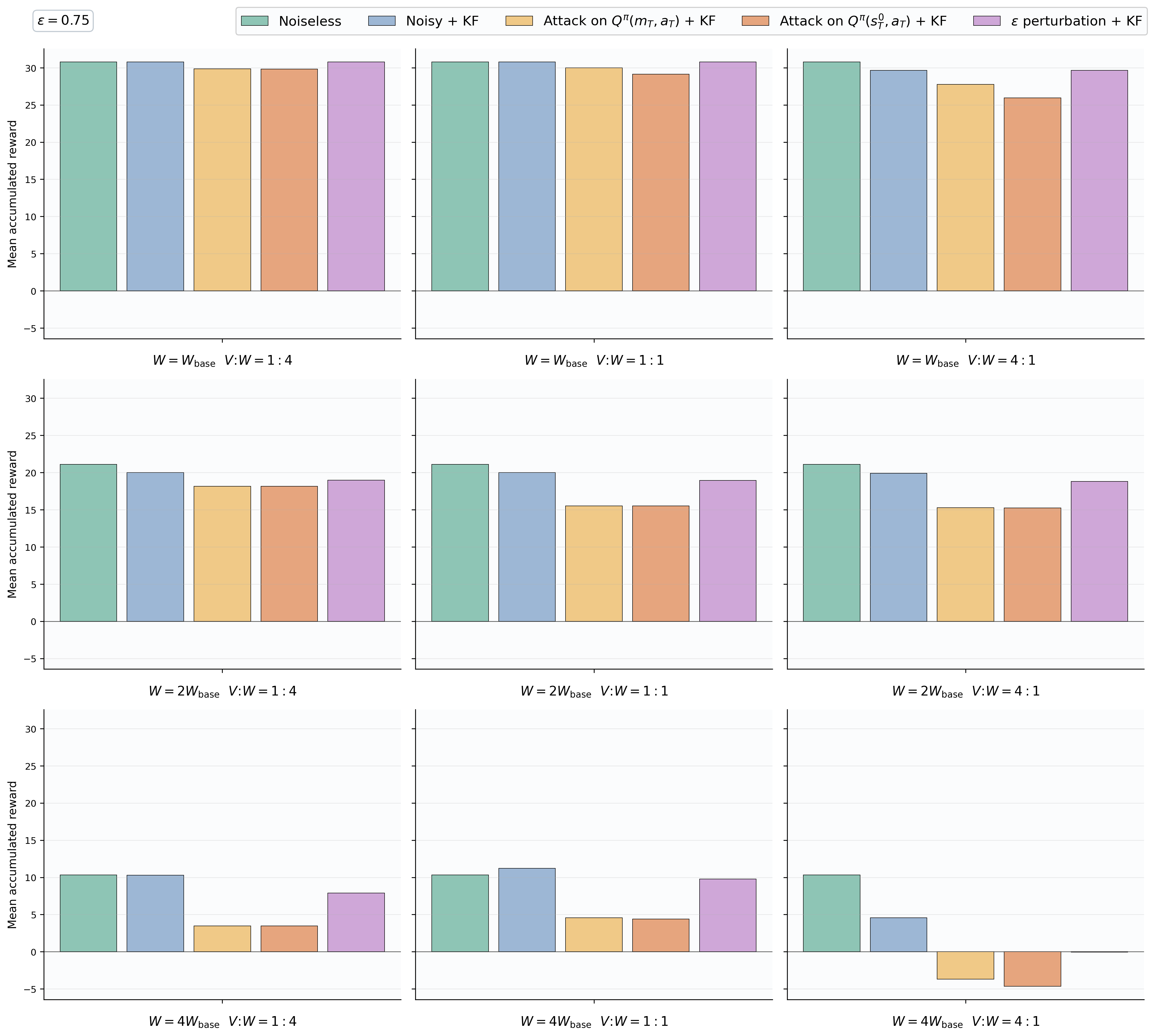}
    \caption{Mean cumulative reward over \(100\) episodes under different
    uncertainty configurations for constraints of $\epsilon = 0.95$.}
    \label{fig:75covaraincesweep}
\end{figure}

\begin{figure}[ht]
    \centering
    \includegraphics[width=0.95\linewidth]{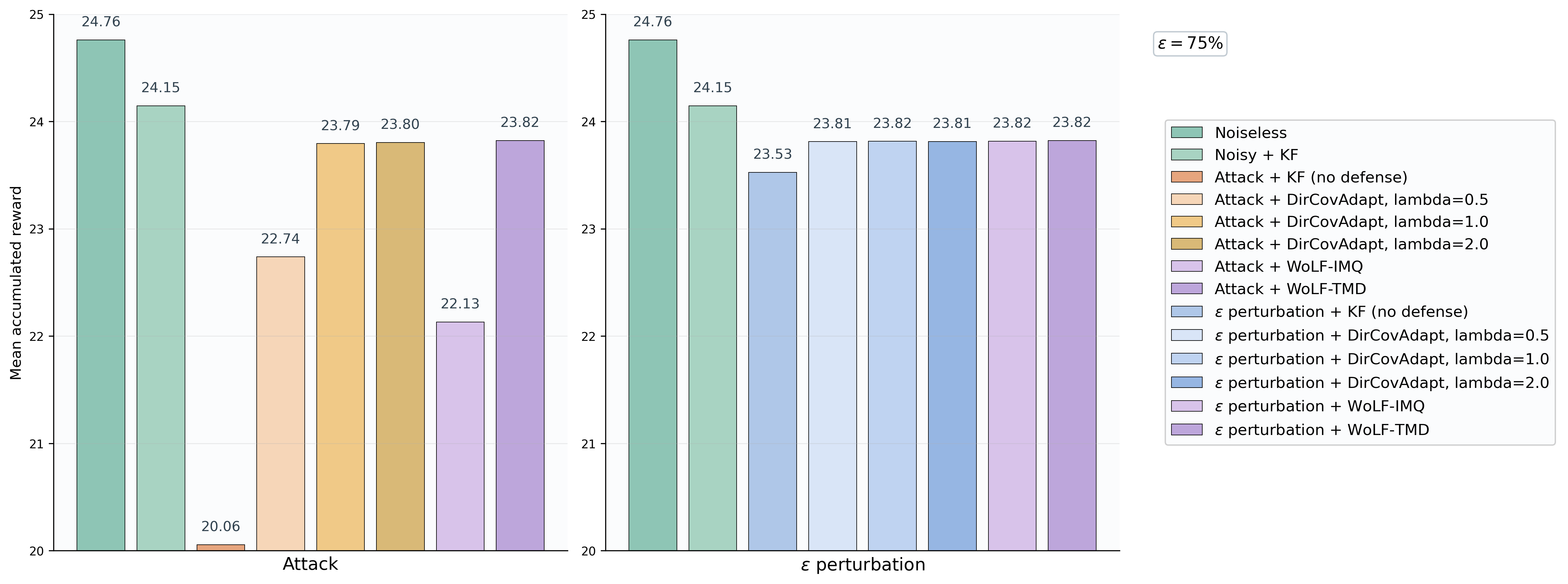}
    \caption{Mean cumulative reward over \(250\) episodes ($\epsilon=0.75$). \textbf{Left:} Performance of the \emph{noiseless}, \emph{noisy + KF}, and \emph{attack + KF} settings, together with the proposed directional covariance-adaptation defense and the WoLF baselines under adversarial attacks. \textbf{Right:} Performance of the \emph{noiseless}, \emph{noisy + KF}, and \emph{\(\epsilon\)-perturbation + KF} settings, together with the proposed directional covariance-adaptation defense and the WoLF baselines under random \(\epsilon\)-perturbations.}
    \label{fig:defense_benchmark75.png}
\end{figure}

\end{document}